%% file: main.tex
\documentclass[11pt]{yresearch}
\usepackage{amsmath,amssymb,amsfonts,amsthm}
\newtheorem{theorem}{Theorem}
\usepackage{enumitem}
\usepackage{algorithm}
\usepackage{algorithmic}

\usepackage{pifont}

\graphicspath{{figures/}}

\definecolor{tradeoff}{HTML}{2196F3}
\definecolor{redundant}{HTML}{F44336}
\definecolor{neutral}{HTML}{9E9E9E}

\title{\huge BenchScope: How Many Independent Signals Does Your Benchmark Provide?}

\author[1]{Tommy Sha\textsuperscript{*}}
\author[2]{Stella Zhao\textsuperscript{*}}
\affiliation[1]{Stony Brook University}
\affiliation[2]{University of Minnesota Twin Cities}
\correspondence{\email{tianming.sha@stonybrook.edu}, \email{zhao2052@umn.edu}}
\date{March 2026}

\contribution[\textsuperscript{*}]{Equal contribution; authors listed in alphabetical order.}

\abstract{AI evaluation suites often report many scores without checking whether those scores carry independent information.
We introduce \textit{Effective Dimensionality} (ED), the participation ratio of a centered benchmark-score spectrum, as a fast, population-conditional upper-bound diagnostic of measurement breadth.
Applied at per-instance granularity to 22 benchmarks across 8 domains and more than 8{,}400 model evaluations, ED reveals substantial redundancy: the six-score Open LLM Leaderboard behaves like roughly two effective measurement axes ($\text{ED} = 1.7$), BBH and MMLU-Pro are near-interchangeable ($\rho = 0.96$, stable across seven subpopulations), and measurement breadth varies more than $20\times$ across current benchmarks.
We show that relative ED rankings are stable under matched-dimension controls and that ED can flag redundant suite components, monitor performance-conditional compression, and guide benchmark maintenance.
Because binary spectra overestimate absolute latent dimensionality, we interpret ED as a screening statistic rather than a literal factor count and complement it with null, reliability, and saturation analyses.
We provide a 22-benchmark reference atlas and a four-step diagnostic workflow that benchmark maintainers can run with a score matrix and a few lines of code.}

\begin{document}
\mymaketitle

\input{sections/introduction}
\input{sections/method_landscape}
\input{sections/landscape}
\input{sections/redundancy_tradeoffs}
\input{sections/coverage}
\input{sections/aging}
\input{sections/semantics}
\input{sections/criteria}
\input{sections/related}
\input{sections/discussion}

\bibliographystyle{plainnat}
\bibliography{references}

\beginappendix
\input{appendix/data_details}
\input{appendix/additional}

\end{document}

%% file: sections/introduction.tex
\section{Introduction}
\label{sec:introduction}

AI benchmarks are proliferating faster than ever, yet there is no principled method to assess whether they actually provide independent information.
Benchmark designers list task categories and assume diversity; leaderboard operators average scores and assume comprehensiveness; practitioners compare composite rankings and assume informativeness.
None of these assumptions are tested.
The result: massive hidden redundancy, wasted evaluation compute, and rankings that mislead more than they inform.

Consider the Open LLM Leaderboard v2~\citep{open-llm-leaderboard-v2}, which reports six benchmark scores per model.
How many of those scores carry independent information?
A spectral analysis reveals that the suite behaves like roughly two effective measurement axes: $\text{ED} = 1.7$ out of a maximum of~6 (ranging from 1.49 to 2.20 across seven model subpopulations).
BBH and MMLU-Pro are nearly the same signal ($\rho = 0.96$, stable across all seven subpopulations tested, $\rho \in [0.88, 0.97]$), yet both are counted as separate benchmarks, double-weighting a single axis in every composite ranking.
Under random benchmark weightings, the champion changes 38\% of the time; the uniform-weight leader places 441st out of 4{,}576 on MATH.
The leaderboard is averaging near-redundant signals, producing a compromise ranking rather than a clean measure of any one capability.

This is not an isolated case.
Across the evaluation ecosystem, benchmark designers' claimed breadth systematically overstates what their benchmarks actually measure.
The Berkeley Function Calling Leaderboard (BFCL) advertises 20 function-calling categories, but they collapse to just 7 effective groups; \texttt{simple\_python}, \texttt{simple\_java}, and \texttt{simple\_javascript} test the same underlying skill ($r = 0.84$).
MMLU-Pro's 14 subject categories (biology, chemistry, economics, \ldots) yield a category-level ED of 1.1, nearly perfectly unidimensional despite spanning the sciences and humanities.
At the other extreme, BigCodeBench achieves $\text{ED} = 29$ under a single label (practical coding), though deeper validation reveals only ${\sim}8$--$9$ statistically significant dimensions even there (Section~\ref{sec:what-drives-ed}).
The designers' taxonomy is an unreliable guide to what a benchmark actually measures.

These examples all point to the same gap: there is no fast, computable diagnostic for how much independent measurement a benchmark suite provides.
To fill this gap, we introduce \textit{Effective Dimensionality} (ED), the participation ratio of a benchmark's centered score spectrum.
ED provides a population-conditional upper bound on measurement breadth: $\text{ED} = 1$ means all tasks rank models the same way; higher values indicate more independent axes of variation.
Because binary spectra systematically overestimate absolute latent dimensionality, ED is best interpreted as a screening statistic for redundancy rather than a literal factor count; we complement it with null-model, reliability, and saturation analyses when count-like precision is needed.
ED is closed-form, requires no threshold or permutation, and plugs directly into a maintainer workflow (Section~\ref{sec:guidelines}).

Applying ED at per-instance granularity to 22 benchmarks spanning 8 domains and more than 8{,}400 model evaluations (${\sim}700$ unique model families), we build the first cross-benchmark redundancy atlas of the AI evaluation ecosystem (Table~\ref{tab:reference}).
The atlas reveals a ${\sim}27\times$ range in raw measurement breadth---from BigCodeBench ($\text{ED} = 29$) to GAIA and EvalPlus ($\text{ED} \approx 1.2$)---with independent measurement capacity scarce across the entire ecosystem.
ED also reveals subtler structural phenomena.
Within a curated 188-model subset, MATH and MuSR are negatively correlated ($\rho = -0.64$), yet across all 4{,}576 models the correlation is positive ($\rho \approx +0.5$), a Simpson's paradox driven by model size as a confounder.
We prove (Theorem~\ref{thm:composite}) that no positive-weight composite can correlate above $0.42$ with both benchmarks within such a subpopulation.
ED is a property of the benchmark--population interaction, not of the benchmark alone, grounded in R\'{e}nyi entropy and random matrix theory and validated against IRT and five alternative dimensionality methods ($\rho = 0.98$; Section~\ref{sec:method}).

Despite growing concern about benchmark quality~\citep{raji2021ai,eriksson2025benchmarks,reuel2024betterbench}, prior work has been predominantly qualitative, identifying what can go wrong but not providing computable diagnostics for specific benchmarks.
Two recent studies have begun probing benchmark structure quantitatively: \citet{federiakin2025psychometric} applied Confirmatory Factor Analysis to IFEval and \citet{kearns2026structure} applied PCA to the Open LLM Leaderboard.
However, both operate at aggregate granularity (per-category or per-benchmark scores), and we show that this granularity inflates apparent dimensionality by up to $5\times$; only per-instance analysis cleanly separates genuine capability dimensions from task-difficulty variation.

Our contributions:
\begin{enumerate}[leftmargin=1.8em, itemsep=3pt, label=(\arabic*)]
    \item \textbf{ED: a fast, closed-form screening diagnostic for benchmark redundancy.}
    We define effective dimensionality via the participation ratio, connect it to R\'{e}nyi entropy and random matrix theory, and validate it against IRT and five alternative methods ($\rho = 0.98$ rank-order agreement on synthetic data; Appendix~\ref{sec:app-additional}).
    ED provides a population-conditional upper bound on measurement breadth that requires no thresholds, no permutations, and no iterative fitting.

    \item \textbf{A 22-benchmark redundancy atlas} (Table~\ref{tab:reference}) exposing a $>$20$\times$ range in measurement breadth and revealing that independent measurement capacity is scarce even for the broadest benchmarks.

    \item \textbf{Design principles and an actionable diagnostic workflow.}
    We identify that structural heterogeneity---not topic diversity---drives measurement breadth (Section~\ref{sec:coverage}), and distill this into a four-step diagnostic workflow (Section~\ref{sec:guidelines}) with three empirical heuristics that benchmark maintainers can apply with only a score matrix and a few lines of code.

    \item \textbf{Supporting theory and case studies.}
    A composite correlation ceiling (Theorem~\ref{thm:composite}); case studies on benchmark aging, ED-guided task selection, and compression cost ($\rho = 0.64$ across 8 benchmarks spanning 6 domains); and the first systematic quantification of conditional negative correlations across benchmark pairs.
\end{enumerate}

\noindent
\textbf{Evidence calibration.}
The redundancy and ranking-fragility results hold across all seven subpopulations tested; the conditional negative correlations are population-specific (Section~\ref{sec:redundancy-tradeoffs}).

\noindent We define ED and its theoretical foundations (Section~\ref{sec:method}), present the 22-benchmark atlas (Section~\ref{sec:landscape}), dissect the Open LLM Leaderboard as a case study (Section~\ref{sec:redundancy-tradeoffs}), investigate what drives ED and how it changes over time (Sections~\ref{sec:coverage}--\ref{sec:aging}), validate the semantic content of principal components (Section~\ref{sec:semantics}), and distill a practical diagnostic workflow (Section~\ref{sec:guidelines}).

\begin{figure}[htbp]
\centering
\includegraphics[width=\textwidth]{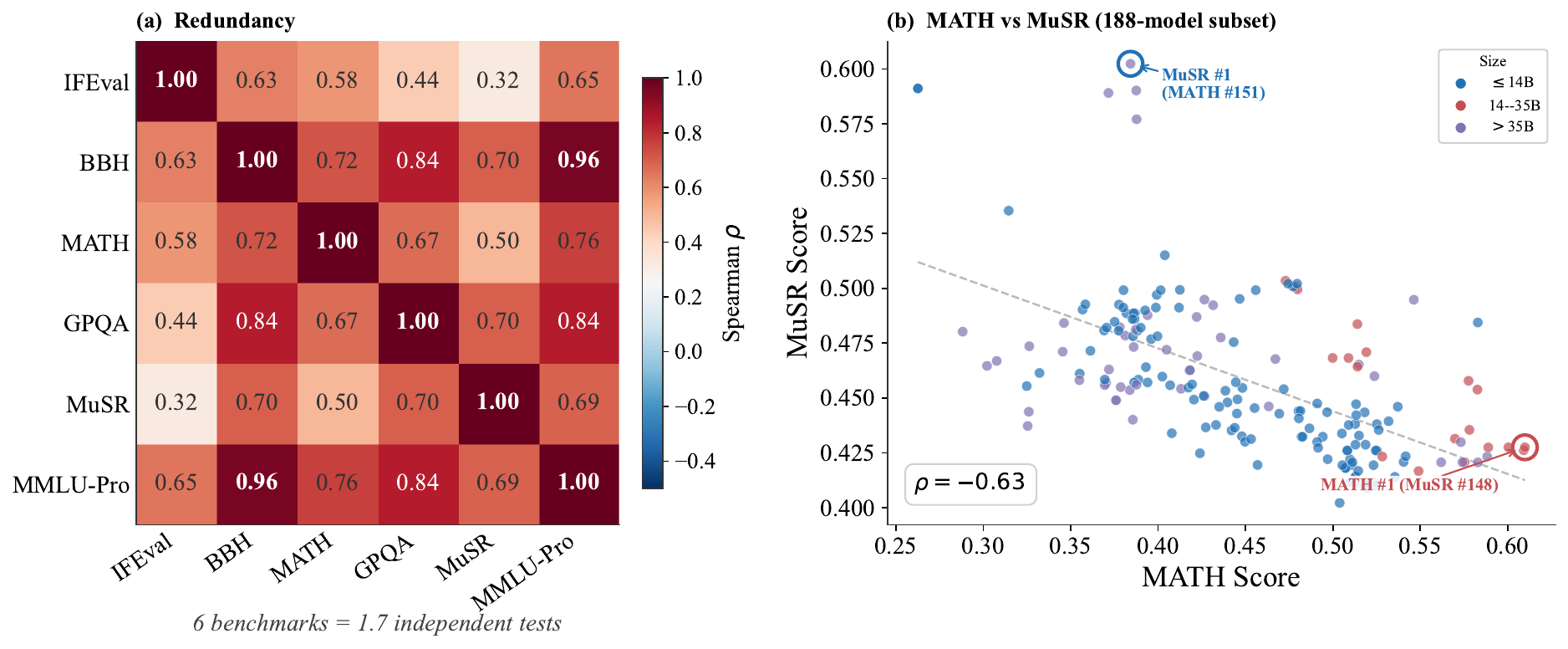}
\caption{Redundancy structure in the Open LLM Leaderboard v2.
(a)~The pairwise correlation matrix (full 4{,}576-model population) shows that the six-score suite behaves like roughly two effective measurement axes ($\text{ED} = 1.7$): BBH and MMLU-Pro are near-duplicates ($\rho = 0.96$) while other pairs are only weakly correlated (e.g., IFEval--MuSR: $\rho = 0.32$).
(b)~A subset-specific pattern ($n = 188$ only): within a curated subset, MATH and MuSR are negatively correlated ($\rho = -0.64$); in the full population, the correlation is positive ($\rho = 0.50$; see Figure~\ref{fig:tradeoffs}).}
\label{fig:teaser}
\end{figure}

%% file: sections/method_landscape.tex
\section{Effective Dimensionality: Definition and Theory}
\label{sec:method}

Benchmarks are designed to evaluate models, but the arrow can be reversed.
Given a population of models, their collective pass/fail patterns on a benchmark's tasks reveal the benchmark's latent structure: how many independent axes of capability variation it actually measures.
We formalize this reversal: a binary pass/fail matrix, spectrally decomposed, yields a single scalar that answers this question.

\subsection{Definition}
\label{sec:ed-definition}

Operationally, we encode pass/fail patterns in a matrix $M \in \{0,1\}^{T \times N}$, with $T$ tasks as rows and $N$ models as columns.
We task-center $M$ by subtracting each task's mean pass rate across models, removing trivial difficulty effects before computing the SVD, which yields singular values $\sigma_1 \geq \sigma_2 \geq \cdots \geq 0$.
The resulting covariance structure captures how models differ from each other on each task, rather than reflecting which tasks are globally easy or hard.
The Effective Dimensionality (ED) is the participation ratio of the squared spectrum:
\begin{equation}
  \boxed{\;\text{ED} = \frac{\bigl(\sum_i \sigma_i^2\bigr)^2}{\sum_i \sigma_i^4}\;}
  \label{eq:ed}
\end{equation}
When a single principal component captures all variance, $\text{ED} = 1$; when $K$ components share variance equally, $\text{ED} = K$.
The participation ratio originates in condensed-matter physics~\citep{thouless1974,bell1970} and computational neuroscience~\citep{gao2017,stringer2019}; to our knowledge this is its first application to AI benchmark analysis.

\textbf{Interpretation.} ED is a population-conditional upper bound on measurement breadth, not a literal count of independent capabilities.
``SWE-bench Verified $\text{ED} = 7.9$'' means that given the current 134 agents, the benchmark's discriminative structure spans at most about eight effective axes; permutation-null and split-half analyses (Section~\ref{sec:what-drives-ed}) can refine this to a count of statistically significant or reproducible dimensions.
Like test dimensionality in psychometrics~\citep{embretson2000irt}, ED is a property of the benchmark--population interaction and should be recomputed as the model ecosystem evolves.

\begin{figure}[htbp]
\centering
\includegraphics[width=\textwidth]{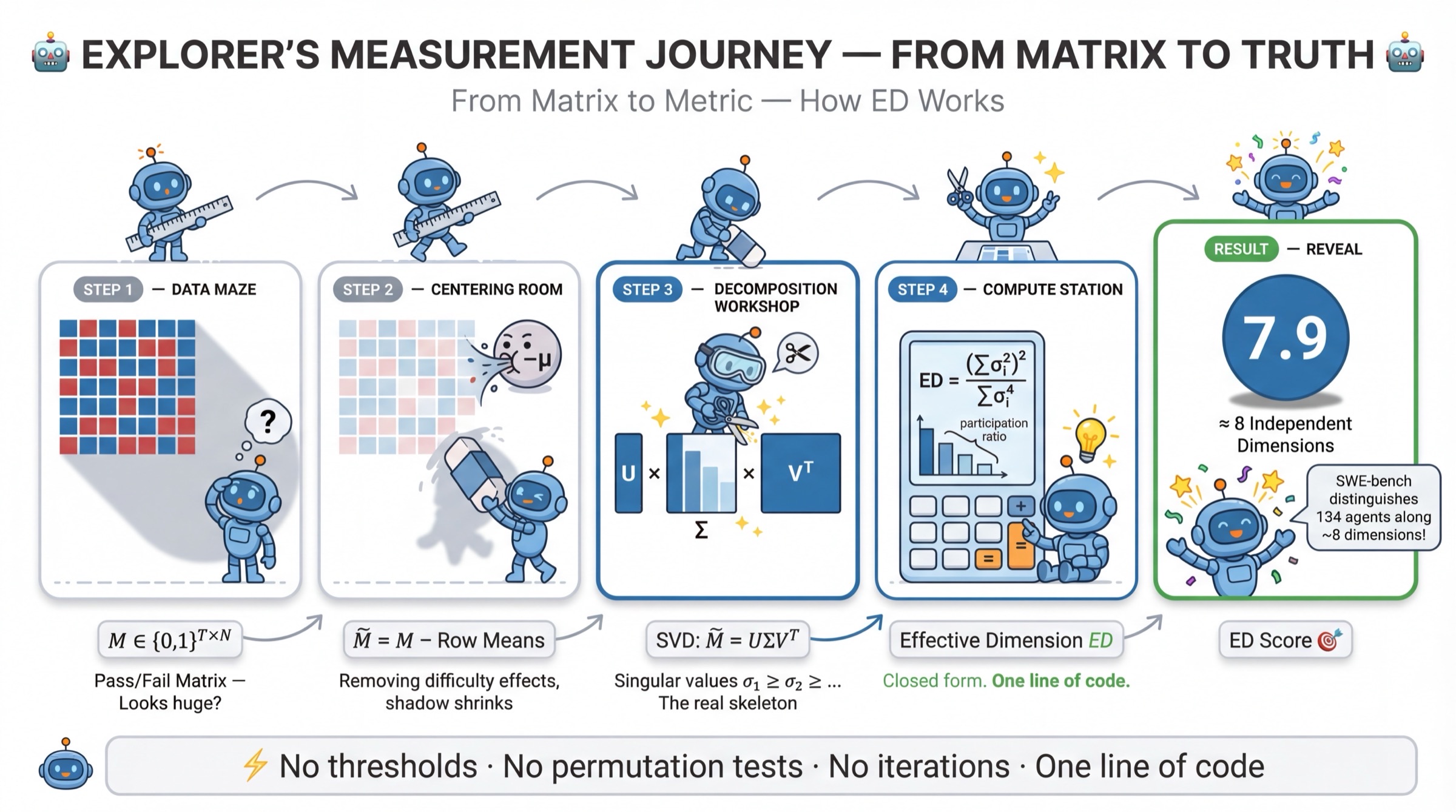}
\caption{From matrix to metric: ED computation pipeline.
A binary pass/fail matrix is task-centered, spectrally decomposed via SVD, and summarized by the participation ratio, a single closed-form scalar requiring no thresholds, no permutations, and no iterations.}
\label{fig:pipeline}
\end{figure}

\subsection{Information-Theoretic Interpretation}
\label{sec:info-theory}

ED admits a natural information-theoretic interpretation.
Define the normalized eigenvalue spectrum $\lambda_i = \sigma_i^2 \big/ \sum_j \sigma_j^2$ as a probability distribution over principal components.
The participation ratio corresponds to the R\'{e}nyi entropy of order~2:
\begin{equation}
  \text{ED} = \exp\!\bigl(H_2(\boldsymbol{\lambda})\bigr) = \frac{1}{\sum_i \lambda_i^2}\,,
  \label{eq:ed-renyi}
\end{equation}
where $H_2(\boldsymbol{\lambda}) = -\log\sum_i \lambda_i^2$ is the collision entropy~\citep{renyi1961}, a natural measure of how many components effectively participate in representing model variation.

The R\'{e}nyi connection provides two insights.
First, ED is bounded: $1 \leq \text{ED} \leq \min(T,N)$, because the SVD of a $T \times N$ matrix produces at most $\min(T,N)$ nonzero singular values.
When $N \gg T$ (e.g., GAIA: $3 \times 3{,}098$), the upper bound is $T = 3$, so low ED in such cases may reflect insufficient task granularity rather than genuinely one-dimensional measurement.
Second, because $H_2 \leq H_1$ for any distribution, ED $= \exp(H_2)$ is always less than or equal to $\exp(H_1)$, the Shannon-entropy-based effective rank~\citep{roy2007effective}.
ED is therefore more sensitive to dominant components than its Shannon counterpart, an appropriately calibrated default for benchmark diagnostics, where false positives in dimensionality inflate confidence in suite breadth.

\subsection{Random Matrix Theory Baseline}
\label{sec:rmt-baseline}

To calibrate whether an observed ED is genuinely low, we need a null baseline.
Under the hypothesis that all pass/fail entries are independent (no latent capability structure), the Marchenko--Pastur law~\citep{marchenko1967} governs the eigenvalue distribution.
For a $T \times N$ matrix with i.i.d.\ centered entries, applying the participation-ratio formula to the expected trace moments~\citep{marchenko1967,bai2010spectral} yields:
\begin{equation}
  \text{ED}_{\text{null}} \;=\; \frac{TN}{T + N}\,.
  \label{eq:ed-null}
\end{equation}
This closed-form expression depends only on matrix dimensions, not on entry variance (which cancels).
Because actual pass/fail matrices are column-centered, heterogeneous, and binary, this formula serves as an analytical baseline rather than an exact expected value; empirically, it agrees within 10\% of simulation-based permutation nulls across all 22 benchmarks.

Table~\ref{tab:ed-rmt} applies this formula to three benchmarks spanning the ED range.
For SWE-bench Verified ($T{=}500$, $N{=}134$), $\text{ED}_{\text{null}} = 105.7$ while $\text{ED}_{\text{obs}} = 7.9$; the observed benchmark retains only 7.5\% of the dimensionality expected under random structure.
Across all 22 benchmarks, the ratio $\text{ED}_{\text{obs}} / \text{ED}_{\text{null}}$ ranges from $0.015$ (BFCL) to $0.81$ (MuSR), with a median of $0.21$; every benchmark is far more structured than a random matrix.
Bootstrap resampling (1{,}000 iterations) yields 95\% confidence intervals with widths of $0.7$--$3.2$ ED units, confirming statistical stability.

\begin{table}[htbp]
\centering
\caption{Observed ED vs.\ Marchenko--Pastur null (Eq.~\ref{eq:ed-null}). The ratio $\text{ED}_{\text{obs}}/\text{ED}_{\text{null}}$ quantifies how much more structured than random each benchmark is, on a 0--1 scale.}
\label{tab:ed-rmt}
\small
\begin{tabular}{lccccc}
\toprule
\textbf{Benchmark} & $T$ & $N$ & $\textbf{ED}_{\textbf{null}}$ & $\textbf{ED}_{\textbf{obs}}$ & \textbf{Ratio} \\
\midrule
SWE-bench Verified    & 500  & 134  & 105.7 & 7.9  & 0.075 \\
Open LLM (43-sub)     & 43   & 188  & 35.0  & 4.5  & 0.129 \\
BigCodeBench-complete  & 1140 & 153  & 134.9 & 28.9 & 0.214 \\
\bottomrule
\end{tabular}
\end{table}

\paragraph{Robustness.}
\label{sec:robustness}
Three concerns merit attention (full details in Appendix~\ref{sec:app-additional}).
\textit{Centering}: alternative centering schemes shift absolute ED by up to $4\times$ but leave benchmark rankings invariant (Table~\ref{tab:centering}).
\textit{Binary attenuation}: Pearson-based ED overestimates true dimensionality by ${\sim}2\times$ (tetrachoric-corrected $\text{ED}_\text{tet} \approx 0.5\times$ Pearson-ED); we report both in Table~\ref{tab:reference}, using Pearson-ED as the primary metric for its rank-preserving tractability.
Synthetic validation confirms rank-order recovery ($\rho = 0.98$ with true dimensionality).
\textit{Agreement with alternatives}: IRT parallel analysis, broken-stick, Kaiser criterion, and explained-variance thresholds all produce the same benchmark ordering (mean pairwise $\rho = 0.71$; Table~\ref{tab:ed-vs-pa}).

%% file: sections/landscape.tex
\section{A 22-Benchmark Landscape}
\label{sec:landscape}

If benchmark designers cannot rely on their own category taxonomies to gauge measurement breadth, what does the evaluation ecosystem actually look like when measured empirically?
We apply ED to 22 benchmarks spanning 8 evaluation domains and more than 8{,}400 models, building the first cross-benchmark redundancy atlas.
The result reveals a wide range of measurement breadth: benchmarks that appear equally comprehensive by design-time criteria differ by more than $20\times$ in effective dimensionality.


\begin{table}[!ht]
\centering
\caption{Effective dimensionality of 22 benchmarks. ED is the Pearson-based upper bound; $\text{ED}_\text{tet}$ is the tetrachoric-corrected estimate (${\approx}0.5\times$ ED). \textbf{Top}: unconstrained. \textbf{Bottom}: shape-constrained ($\dagger$, $\min(T,N) \leq 6$).
Relative ordering survives matched-dimension controls ($\rho = 1.0$ at $300 \times 50$).}
\label{tab:reference}
\footnotesize\setlength{\tabcolsep}{4pt}
\begin{tabular}{llcrrrr}
\toprule
\textbf{Benchmark} & \textbf{Domain} & \textbf{Size ($T\!\times\!N$)} & \textbf{ED} & \textbf{ED$_\text{tet}$} & \textbf{PC1\%} & \textbf{ED/ED$_\text{n}$} \\
\midrule
\multicolumn{7}{l}{\textit{Unconstrained ($\min(T,N) > 6$):}} \\
BigCodeBench-instruct       & Coding (prac.)  & $1140\!\times\!126$    & 33.0 & ${\sim}$17 & 14.7 & .291 \\
BigCodeBench-complete       & Coding (prac.)  & $1140\!\times\!153$    & 28.9 & ${\sim}$16 & 16.5 & .214 \\
BigCodeBench-hard-comp.     & Coding (prac.)  & $148\!\times\!199$     & 18.5 & ${\sim}$10 & 20.6 & .218 \\
BigCodeBench-hard-inst.     & Coding (prac.)  & $148\!\times\!173$     & 18.3 & ${\sim}$10 & 20.3 & .229 \\
SWE-bench Verified          & Coding (repo)   & $500\!\times\!134$     & 7.9  & ${\sim}$4 & 34.4 & .075 \\
MMLU-Pro                    & Knowledge       & $12257\!\times\!47$    & 6.8  & ${\sim}$4 & 36.4 & .145 \\
LiveCodeBench               & Coding (cont.)  & $1055\!\times\!52$     & 6.2  & ${\sim}$3 & 37.0 & .126 \\
LiveBench                   & Comprehensive   & $494\!\times\!195$     & 5.9  & ${\sim}$3 & 33.5 & .042 \\
Open LLM (43-sub)           & LLM eval        & $43\!\times\!188$      & 4.5  & ${\sim}$2 & 41.2 & .129 \\
BBH                         & Reasoning       & $24\!\times\!188$      & 4.4  & ${\sim}$2 & 41.8 & .206 \\
SWE-bench Test              & Coding (repo)   & $1839\!\times\!24$     & 3.8  & ${\sim}$2 & 49.6 & .161 \\
SWE-bench Lite              & Coding (repo)   & $299\!\times\!84$      & 3.7  & ${\sim}$2 & 51.5 & .056 \\
SWE-bench Multilingual      & Coding (multi.) & $301\!\times\!13$      & 2.9  & ${\sim}$2 & 57.6 & .233 \\
SWE-bench Multimodal        & Coding (GUI)    & $301\!\times\!12$      & 2.0  & ${\sim}$1 & 69.7 & .173 \\
BFCL                        & Func.\ call.    & $1000\!\times\!109$    & 1.5  & ${\sim}$1 & 81.0 & .015 \\
\midrule
\multicolumn{7}{l}{\textit{Shape-constrained ($\min(T,N) \leq 6$; ED may underestimate true dimensionality):}} \\
MuSR$^\dagger$              & Reasoning       & $3\!\times\!188$       & 2.4  & ${\sim}$1 & 49.4 & .813 \\
AgentBench$^\dagger$        & General agent   & $5\!\times\!25$        & 2.2  & ${\sim}$1 & 60.0 & .528 \\
Open LLM (6-bench)$^\dagger$ & LLM eval       & $6\!\times\!4576$      & 1.7  & ${\sim}$1 & 75.6 & .284 \\
GPQA$^\dagger$              & Science         & $3\!\times\!188$       & 1.4  & ${\sim}$1 & 83.3 & .474 \\
MATH$^\dagger$              & Mathematics     & $7\!\times\!188$       & 1.3  & ${\sim}$1 & 88.5 & .193 \\
GAIA$^\dagger$              & General agent   & $3\!\times\!3098$      & 1.2  & ${\sim}$1 & 91.1 & .400 \\
EvalPlus$^\dagger$          & Coding (func.)  & $4\!\times\!125$       & 1.2  & ${\sim}$1 & 90.8 & .310 \\
\bottomrule
\end{tabular}
\end{table}

Table~\ref{tab:reference} provides the first cross-benchmark ED reference atlas covering 22 benchmarks across 8 domains and 8{,}400+ models, calibrated to the current evaluation ecosystem.
These benchmarks use different model populations, matrix sizes, and task granularities; the sort order is for convenience and the absolute values are population-conditional.
The $\text{ED}_\text{tet}$ column provides a tighter (lower) estimate correcting for binary attenuation, and $\text{ED}/\text{ED}_\text{n}$ normalizes by the Marchenko--Pastur null (Eq.~\ref{eq:ed-null}), providing a scale-free comparison that accounts for matrix shape.

\paragraph{Three tiers.}
\textbf{Low ED ($<$4):} BFCL, SWE-bench Test/Lite/Multilingual/Multimodal have homogeneous task structure dominated by a single ability axis.
\textbf{Medium ED (4--8):} SWE-bench Verified, MMLU-Pro, LiveCodeBench, LiveBench, and BBH have moderate multi-dimensional discrimination.
\textbf{High ED ($>$18):} BigCodeBench stands alone, and its lead survives matched-dimension control ($\text{ED} = 19.7$ at $300 \times 50$), confirming that it genuinely measures more independent capabilities.
Benchmarks marked $\dagger$ have $\min(T,N) \leq 6$, imposing a hard ceiling on ED; their reported values serve as lower bounds on true dimensionality, and we flag these cases explicitly in Table~\ref{tab:reference}.

\paragraph{Broad vs.\ focused benchmarks.}
Under ED, a broad comprehensive benchmark is one that discriminates models along many independent axes (high ED).
BigCodeBench ($\text{ED} = 29$) leads our corpus: 93\% of its task pairs are weakly correlated ($|r| < 0.3$), suggesting genuinely diverse measurement.
However, deeper validation (Section~\ref{sec:what-drives-ed}) reveals that only ${\sim}8$--$9$ of these dimensions exceed the permutation null; even the best benchmark in our corpus has far fewer real dimensions than its raw ED suggests.
At the other extreme, BFCL ($\text{ED} = 1.5$) advertises 20 categories, yet a single ability axis dominates 81\% of the variance.
For focused evaluations (e.g., MATH, $\text{ED} = 1.3$), low ED is a design strength: the benchmark measures one thing well.
The problem arises when a low-ED benchmark is used as if it were comprehensive, or when a benchmark suite aggregates multiple low-ED components that happen to measure the same axis.
The deeper lesson: \textbf{independent measurement capacity is scarce across the entire evaluation ecosystem}, even for benchmarks that appear richly multi-dimensional.

\begin{figure}[htbp]
\centering
\includegraphics[width=\textwidth]{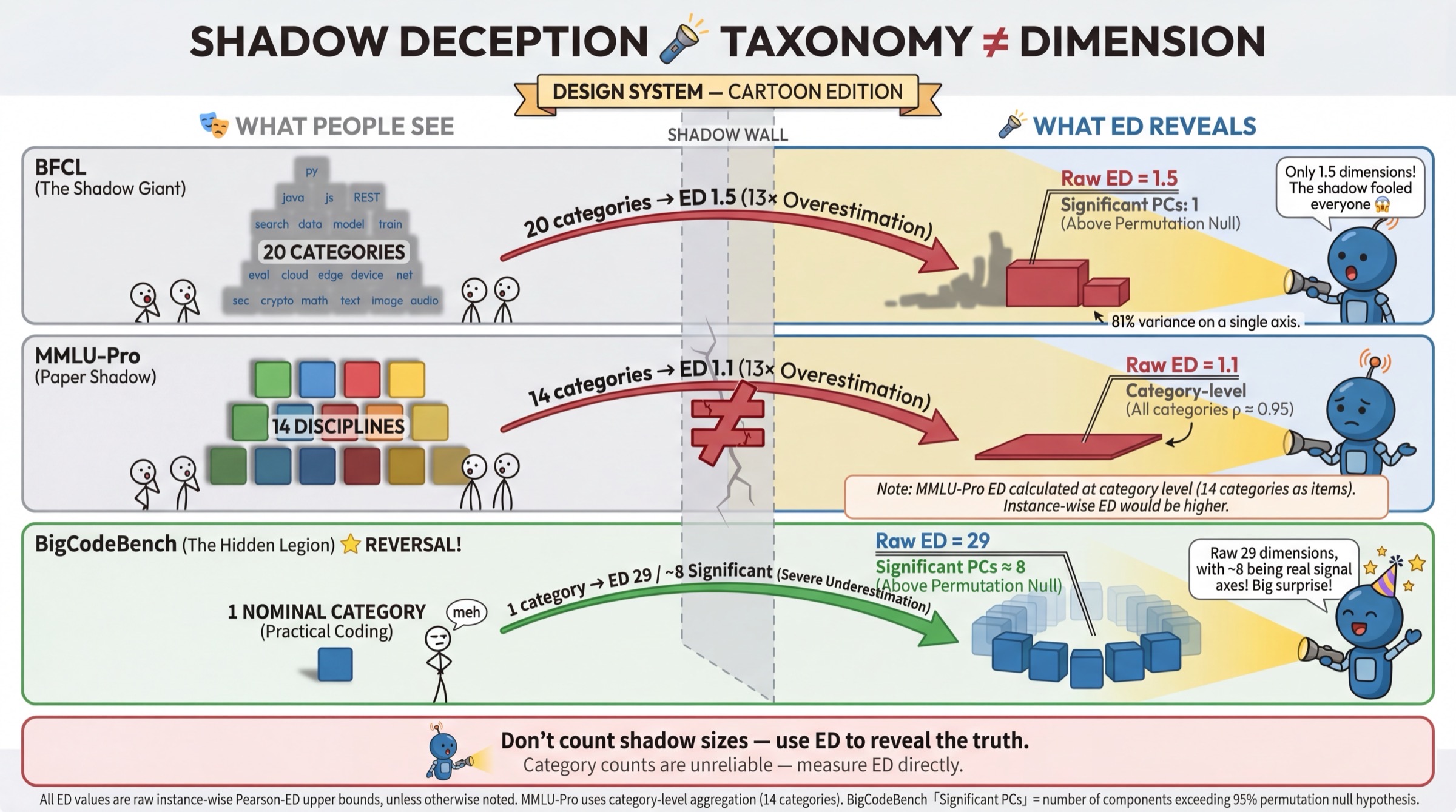}
\caption{Taxonomy $\neq$ measurement breadth.
All three examples use per-instance ED as the common metric.
BFCL's 20 categories yield per-instance ED\,=\,1.5; MMLU-Pro's 14 disciplines yield category-level ED\,=\,1.1; BigCodeBench's single label yields ED\,=\,29 (raw), of which ${\sim}8$ PCs exceed the permutation null.
Category counts can overstate or understate empirical breadth by an order of magnitude.}
\label{fig:taxonomy}
\end{figure}

\paragraph{Robustness of the ranking.}
To verify that the rank ordering is a benchmark property rather than a population artifact, we performed two controls.
First, saturation curves on all six per-instance benchmarks show that five of six are above 86\% of their asymptotic $\text{ED}_\infty$; only BigCodeBench is below 90\% ($\text{ED}_\infty = 35.4$, $n_{1/2} = 35$), consistent with its richer latent structure requiring more models to reveal.
Second, we subsampled all six benchmarks to a matched $300 \times 50$ matrix (30 bootstrap trials): the rank ordering is perfectly preserved ($\rho = 1.0$).
BigCodeBench's ED drops from 28.9 to 19.7 at matched dimensions, SWE-bench from 7.9 to 7.5, and BFCL from 1.5 to 1.5. Absolute values change with population size, but the ordering BigCodeBench $\gg$ SWE-bench $>$ LiveCodeBench $>$ BFCL is invariant across all benchmarks tested.\footnote{The matched-dimension control was run on six benchmarks with per-instance pass/fail data, including IFEval and AlpacaEval (analyzed separately from Table~\ref{tab:reference}; see Section~\ref{sec:ed-greedy}).}
Absolute ED values are population-conditional; relative comparisons across benchmarks are robust.

\paragraph{Granularity matters.}
One methodological warning: analysis granularity systematically affects results.
BFCL at per-category granularity shows PC2$= 26\%$; at per-instance granularity it shrinks to $5\%$, a fivefold drop that prior aggregate-level analyses~\citep{kearns2026structure,federiakin2025psychometric} may share.
All within-benchmark analyses here use per-instance data; the suite-level analysis in Section~\ref{sec:redundancy-tradeoffs} necessarily uses aggregate scores.
For the 7 continuous-score benchmarks (of 22), we binarize at $0.5$; absolute ED varies up to $2.5\times$ with threshold, but tier placement is robust.
Missing entries ($<$5\%) are imputed with column means.

\begin{figure}[htbp]
\centering
\includegraphics[width=\textwidth]{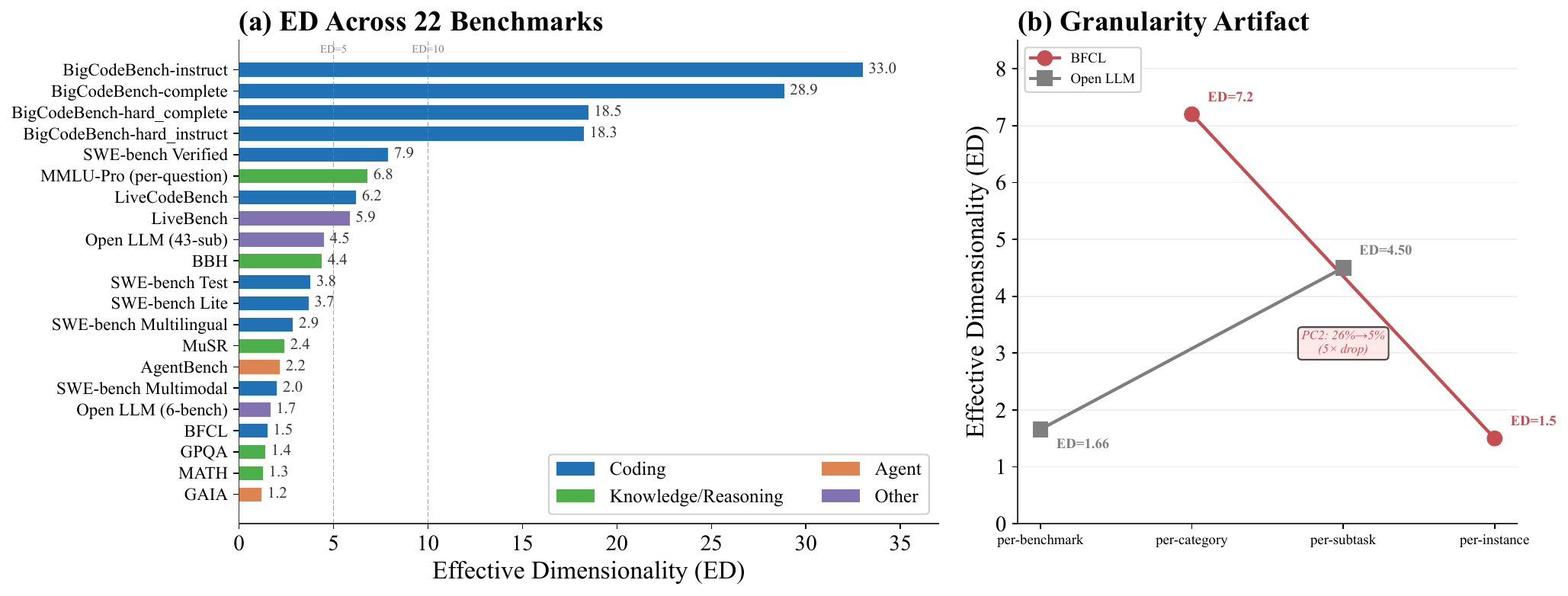}
\caption{ED across 22 benchmarks, sorted by descending ED. BigCodeBench ($\text{ED}=29$) spans more than $20\times$ the raw dimensionality of GAIA and EvalPlus ($\text{ED}=1.2$). The variation is not explained by benchmark size alone; task-set design and matrix shape both contribute (Section~\ref{sec:what-drives-ed}).}
\label{fig:landscape}
\end{figure}

%% file: sections/redundancy_tradeoffs.tex
\section{Case Study: The Open LLM Leaderboard v2}
\label{sec:redundancy-tradeoffs}

We now examine the Open LLM Leaderboard v2~\citep{open-llm-leaderboard-v2} in detail.
It evaluates models on six benchmarks (IFEval, BBH, MATH, GPQA, MuSR, MMLU-Pro); across 4{,}576 models, the $6 \times 6$ aggregate score matrix yields $\text{ED} = 1.66$.
At finer granularity (43 per-subtask scores, 188 models), $\text{ED} = 4.5$, higher but qualitatively unchanged: severe redundancy.

\subsection{Pervasive Redundancy}
\label{subsec:redundancy}

A suite-level ED of 1.66 implies that some of the six scores must carry highly overlapping information.
The clearest overlap is between BBH and MMLU-Pro, whose Spearman correlation is $0.96$ (95\% CI $[0.95,\;0.97]$).
Within this model population, reporting both largely duplicates the same ranking signal.
The BBH--MMLU-Pro redundancy is robust across all seven subpopulations tested (Qwen, Llama, open-source, and size-stratified partitions), with Spearman $\rho$ ranging from $0.883$ (${\geq}$30B) to $0.971$ (Qwen family), confirming this as a fundamental property of the benchmark pair, not a population artifact.

A leave-one-out experiment quantifies the asymmetry: removing MMLU-Pro changes ED by less than $0.02$, confirming that everything it measures is already captured by the remaining five benchmarks.
Removing IFEval changes ED by $0.30$, the largest shift of any single benchmark, making it the only irreplaceable member of the suite.
IFEval is irreplaceable precisely because it has the weakest correlations with the other benchmarks, contributing the most independent information.

Define the information density of a $k$-benchmark suite as $\text{ID} = \text{ED}/k$.
For the Open LLM Leaderboard v2, $\text{ID} = 1.7/6 = 0.28$: only 28\% of the reported scores carry independent information.
The remaining $6 - 1.7 = 4.3$ scores are redundant with information already captured.

\subsection{Ranking Fragility}
\label{subsec:fragility}

The redundancy has concrete consequences for composite rankings.
The composite ranking double-weights redundant pairs---BBH and MMLU-Pro each count as one signal but carry the same information---and the resulting ranking is fragile under perturbation.

\paragraph{Counterfactual analysis.}
Removing MMLU-Pro (the most redundant benchmark, $\rho = 0.96$ with BBH) yields Kendall $\tau = 0.946$ versus the full-suite ranking. High overall, but the champion changes: the leader under six benchmarks drops to fourth, replaced by a Qwen2.5-72B variant.
Nine of the top-10 models remain in the top 10, but 1{,}477 models shift by more than 100 ranks.
Keeping only three maximally informative benchmarks (IFEval, MATH, MuSR, chosen for mutual independence) yields $\tau = 0.837$ and only 5 of the original top-10 survive.

\paragraph{Weight sensitivity.}
We sampled 10{,}000 random Dirichlet weight vectors over the six benchmarks.
With $\alpha = 1$ (uniform Dirichlet), the champion changes in 38\% of trials, and 23 distinct models reach the top spot.
The fragility is sensitive to weight concentration: at $\alpha = 10$ (weights clustered near uniform), the rate drops to 12\%; at $\alpha = 0.1$ (extreme weight diversity), it rises to 63\%.
Even under moderate perturbation ($\alpha = 5$), the champion changes 19\% of the time; the top spot is a weighting convention, not a robust property of the benchmark suite.
Perhaps most strikingly, the composite leader places 441st on MATH (out of 4{,}576): it is optimal under equal weighting, but far from the best at any individual capability.

An exhaustive subset search confirms that four benchmarks---IFEval, BBH, MATH, and MMLU-Pro---reproduce the full-suite ranking better than any other 4-benchmark combination ($\tau = 0.943$), while the worst 4-combination achieves only $\tau = 0.764$.
The choice of which benchmarks to include affects the ranking more than the choice of how to weight them.

\subsection{Conditional Negative Correlations}
\label{subsec:conditional-correlations}

The preceding analyses are properties of the full 4{,}576-model population.
The analysis below concerns a curated 188-model subset (all models with publicly available per-subtask scores) that overrepresents well-known models and underrepresents niche checkpoints.
The negative correlations reported below are specific to this subset; all 15 pairs are positive in the full population.

Within this subset, the Spearman rank correlation between MATH and MuSR is $\rho = -0.635$, with a 95\% bootstrap confidence interval of $[-0.714,\;-0.540]$ (10{,}000 iterations).
Partial $\rho$ controlling for parameter count is $-0.625$, virtually unchanged.
Of the 15 benchmark pairs, 6 exhibit significant negative correlations (CI entirely below~0), 2 are redundant (CI entirely above~0.5), and 7 are mixed or weak (Table~\ref{tab:tradeoffs}).

\begin{table}[htbp]
\centering
\caption{Pairwise correlations from the Open LLM Leaderboard v2 ($n = 188$ curated models).
\textbf{Negative}: 95\% bootstrap CI entirely below 0.
\textbf{Redundant}: CI entirely above 0.5.
$\rho$ and CI are unconditional Spearman correlations; partial $\rho$ controls for log parameter count.
These correlations are specific to this 188-model subset; all 15 pairs are positive in the full 4{,}576-model population.}
\label{tab:tradeoffs}
\small
\begin{tabular}{lcccc}
\toprule
\textbf{Benchmark Pair} & $\boldsymbol{\rho}$ & \textbf{95\% CI} & \textbf{Partial $\boldsymbol{\rho}$} & \textbf{Type} \\
\midrule
MATH $\times$ MuSR        & $-0.635$ & $[-0.714,\;-0.540]$ & $-0.625$ & Negative \\
GPQA $\times$ MATH         & $-0.340$ & $[-0.470,\;-0.200]$ & $-0.584$ & Negative \\
IFEval $\times$ MuSR       & $-0.310$ & $[-0.440,\;-0.170]$ & $-0.611$ & Negative \\
IFEval $\times$ MMLU-Pro   & $-0.150$ & $[-0.290,\;-0.010]$ & $-0.479$ & Negative \\
GPQA $\times$ IFEval       & $-0.270$ & $[-0.410,\;-0.130]$ & $-0.568$ & Negative \\
BBH $\times$ MATH          & $-0.210$ & $[-0.350,\;-0.060]$ & $-0.022$ & Negative \\
\midrule
GPQA $\times$ MuSR         & $+0.690$ & $[+0.600,\;+0.770]$ & & Redundant \\
BBH $\times$ MMLU-Pro      & $+0.960$ & $[+0.950,\;+0.970]$ & & Redundant \\
\bottomrule
\end{tabular}
\end{table}

\paragraph{Population-level sign reversal.}
The MATH--MuSR negative correlation reverses at the population level.
Across the full 4{,}576-model Leaderboard, the unconditional MATH--MuSR Spearman correlation is $\rho \approx +0.5$; after controlling for log model size (3{,}502 models with extractable parameter counts), all 15 pairs remain positive, including MATH$\times$MuSR (partial $\rho = +0.32$).
The population-level sign reversal is largely explained by differences in size composition between the curated 188-model subset and the full Leaderboard.
Within the curated subset, however, the negative association persists after size control (partial $\rho = -0.625$), suggesting additional selection or training-style effects beyond model scale.
Size-stratified analysis quantifies the within-size effect: among Large models (14--35B, $n = 111$), $\rho = -0.682$; among XL models ($>$35B, $n = 50$), $\rho = -0.298$.
The attenuation with scale is consistent with larger models having more capacity to cover multiple objectives.

\paragraph{Composite scoring limits under negative correlation.}
When two benchmarks are negatively correlated within a subpopulation, any positive-weight composite must compromise: tilting toward one necessarily tilts away from the other.
We make this intuition precise with a result that depends only on the pairwise correlation $\rho$, not on ED.

\begin{theorem}[Composite correlation ceiling]
\label{thm:composite}
Let $s_1, s_2$ be standardized benchmark scores with Pearson correlation $\rho$.
For any composite $c = w_1 s_1 + w_2 s_2$ with $w_1, w_2 > 0$:
\begin{equation}
  \max_{w_1, w_2 > 0}\;\min\!\big(\mathrm{corr}(c, s_1),\;\mathrm{corr}(c, s_2)\big) \;=\; \sqrt{\frac{1 + \rho}{2}}\,.
  \label{eq:composite-ceiling}
\end{equation}
The maximum is attained at equal weights $w_1 = w_2$.
\end{theorem}

\noindent
\textit{Proof sketch.}
Setting $w_1 = 1$, the correlations $r_1(t)$ and $r_2(t)$ satisfy $r_1(t) = r_2(1/t)$ by symmetry, so they cross at $t = 1$.
The min is maximized at this crossing point, yielding $\sqrt{(1+\rho)/2}$.
(Full proof in Appendix~\ref{sec:app-additional}.)

\smallskip
\noindent
The 0.42 ceiling for MATH$\times$MuSR ($\rho = -0.64$) means that within the 188-model subset, a practitioner using a composite leaderboard will, in the best case, get a ranking that correlates only $0.42$ with what either individual benchmark measures.
For the redundant pair BBH$\times$MMLU-Pro ($\rho = +0.96$), the ceiling is $0.99$; the composite loses almost nothing, precisely because the pair is redundant.

\begin{table}[htbp]
\centering
\caption{Composite correlation ceiling (Theorem~\ref{thm:composite}) for selected benchmark pairs within the 188-model subset.}
\label{tab:composite-ceiling}
\small
\begin{tabular}{lcc}
\toprule
\textbf{Pair} & \textbf{Spearman $\boldsymbol{\rho}$} & \textbf{Ceiling $\sqrt{(1{+}\rho)/2}$} \\
\midrule
MATH $\times$ MuSR      & $-0.64$ & $0.42$ \\
GPQA $\times$ MATH       & $-0.34$ & $0.57$ \\
IFEval $\times$ MuSR     & $-0.31$ & $0.59$ \\
GPQA $\times$ IFEval     & $-0.27$ & $0.60$ \\
BBH $\times$ MATH        & $-0.21$ & $0.63$ \\
IFEval $\times$ MMLU-Pro & $-0.15$ & $0.65$ \\
\midrule
BBH $\times$ MMLU-Pro    & $+0.96$ & $0.99$ \\
\bottomrule
\end{tabular}
\end{table}

\paragraph{Practical guidance by use case.}
If you are choosing among models within a fixed size band (e.g., comparing 14--35B open-weight models for deployment): expect MATH and MuSR scores to pull in opposite directions ($\rho \approx -0.68$ in the Large band); monitor both.
If you are comparing models across a wide size range (e.g., deciding whether to upgrade from 7B to 70B): the correlation is positive ($\rho \approx +0.5$); larger models generally improve on both, and the tension does not apply.
If you are designing a benchmark suite: conditional negative correlations, where stably observed, may signal complementary information within a size class, but verify this on your specific model population.

Two questions follow: what produces the $27\times$ variation in ED across benchmarks, and does the structure hold as models advance?

\begin{figure}[htbp]
\centering
\includegraphics[width=\textwidth]{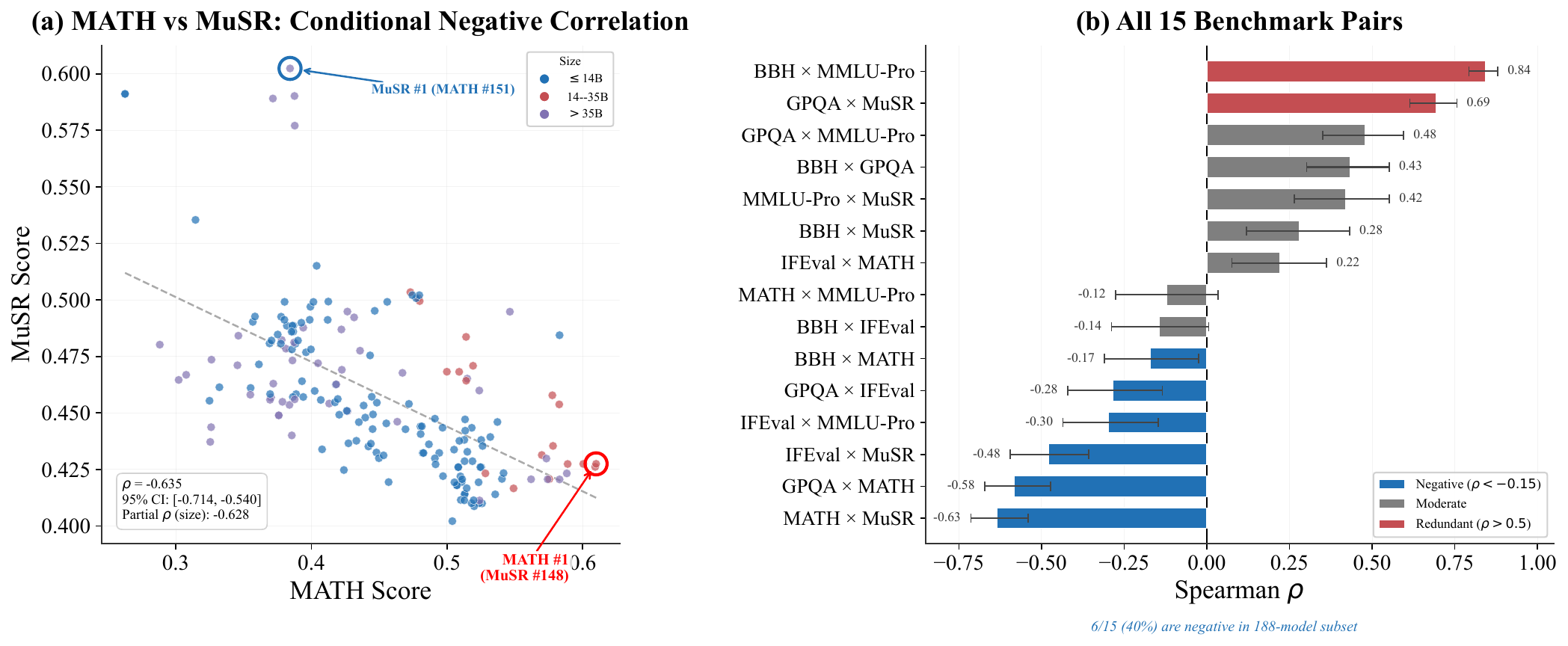}
\caption{Pairwise correlations in the Open LLM Leaderboard v2.
(a)~MATH vs.\ MuSR scatter for the 188-model subset, colored by model size: the negative correlation ($\rho = -0.64$) is visible within size bands, but in the full 4{,}576-model population the correlation is positive ($\rho = +0.50$).
(b)~All 15 benchmark pairs sorted by Spearman $\rho$ with bootstrap 95\% CIs.
Six pairs are significantly negative (blue); two are redundant (red, $\rho > 0.5$).
These negative correlations are specific to the 188-model subset and do not generalize to the full population (Section~\ref{subsec:conditional-correlations}).}
\label{fig:tradeoffs}
\end{figure}

%% file: sections/coverage.tex
\section{What Drives Effective Dimensionality?}
\label{sec:coverage}

The 22-benchmark atlas reveals a $>$20$\times$ range in ED.
What drives this variation?
Benchmark designers typically argue for measurement diversity by listing task categories, but individual case studies show that category counts can be dramatically misleading in both directions.
BFCL claims 20 function-calling categories, yet its per-instance $\text{ED} = 1.5$; MMLU-Pro spans 14 disciplines but achieves category-level $\text{ED} = 1.1$; conversely, BigCodeBench carries a single label yet achieves $\text{ED} = 29$.
(A small-sample aggregate analysis, $\rho = -0.04$, $n = 12$, is consistent with category counts being uninformative but too fragile to bear a strong global conclusion; see Appendix~\ref{sec:app-category-count}.)

\subsection{BFCL: Twenty Categories, Seven Groups}
\label{sec:bfcl-case}

The Berkeley Function Calling Leaderboard (BFCL)~\citep{bfcl} claims to evaluate 20 distinct function-calling capabilities.
We computed the full $20 \times 20$ Spearman correlation matrix across 109 models and applied hierarchical agglomerative clustering (average linkage, distance $= 1 - |r|$).
The 20 categories collapse into seven effective groups.
Most striking: \texttt{simple\_python} $\approx$ \texttt{simple\_java} $\approx$ \texttt{simple\_javascript} ($r = 0.84$), three languages testing the same underlying skill.

The seven groups (Figure~\ref{fig:bfcl}) are:
\begin{enumerate}[leftmargin=1.8em,itemsep=2pt]
  \item \textbf{Multi-function calling} (6 categories): selecting and composing multiple tool invocations.
  \item \textbf{Dialogue / multi-turn} (5 categories): conversational context management.
  \item \textbf{Web search} (2 categories, $r = 0.91$).
  \item \textbf{Format sensitivity} (1 category).
  \item \textbf{Irrelevance detection} (1 category).
  \item \textbf{Live irrelevance} (1 category, negatively correlated with most others).
  \item \textbf{Simple calling} (3 categories, pairwise $r$ up to $0.84$).
\end{enumerate}

\noindent
Hierarchical clustering identifies 7 effective groups out of 20 nominal categories (redundancy ratio $\approx 2.8\times$).
At the finest per-instance granularity (1{,}000 items $\times$ 109 models), BFCL's $\text{ED} = 1.5$, indicating that almost all discriminative variance collapses to a single axis.

MMLU-Pro exhibits an even more extreme case.
Its 14 subject categories (biology, chemistry, computer science, economics, etc.) yield category-level $\text{ED} = 1.1$, nearly perfectly unidimensional despite spanning the sciences, humanities, and social sciences.
The mean inter-category Spearman correlation is $0.95$; the most redundant pair, math$\times$physics, reaches $\rho = 0.99$.
For current LLMs, MMLU-Pro's 14 categories measure a single latent ability (general knowledge retrieval), not 14 independent disciplinary competencies.
If categories overcount (BFCL, MMLU-Pro) and sometimes undercount (BigCodeBench), what actually determines measurement breadth?

\subsection{Structural Heterogeneity, Not Topic Diversity}
\label{sec:what-drives-ed}

Why is BigCodeBench---a benchmark that claims only one task type---the most dimensionally rich in our corpus?
The answer is not topical diversity.
Controlled experiments on BigCodeBench (153 models) show that sampling 100 tasks from a single library (pandas) yields $\text{ED}\!\approx\!18$, nearly identical to sampling across seven libraries ($\text{ED}\!\approx\!19$).
Selecting the 100 tasks whose pass/fail profiles are most similar drops ED to ${\sim}10$, confirming that what matters is how tasks discriminate, not what topics they cover.

The mechanism operates at a deeper level than library diversity.
Library overlap (Jaccard similarity) between task pairs explains negligible variance in pass/fail correlation ($\rho = 0.06$), and restricting to tasks sharing an identical library combination yields the same ED as random samples.
Observable task features---library overlap, difficulty difference, solution length---together explain only 10.6\% of inter-task correlation variance.
The remaining 89\% reflects the logical structure of each task: different algorithmic patterns, control-flow structures, and specification-comprehension demands, even within a single library combination.
Operationalizing this as a measurable design-time feature is a direction for future work.

A three-benchmark comparison quantifies the pattern:

\begin{center}
\small
\begin{tabular}{lccc}
\toprule
& BigCodeBench & SWE-bench & LiveCodeBench \\
\midrule
Task pairs with $|r| < 0.3$ & 93\% & 57\% & 66\% \\
Mean $|r|$ between tasks    & 0.13 & 0.27 & 0.25 \\
\bottomrule
\end{tabular}
\end{center}

\noindent
BigCodeBench's inter-task correlations are dramatically lower because its tasks exercise genuinely different reasoning patterns; SWE-bench tasks share a common cognitive template (read bug report, locate code, edit, test), creating systematically correlated pass/fail patterns.

The practical lesson: \textbf{high ED requires tasks that discriminate models in genuinely different ways, producing approximately independent pass/fail patterns.}
This can be achieved by varying the logical demands of each task (exemplified by BigCodeBench's 1{,}140 tasks spanning diverse algorithmic patterns), not by varying topic labels.

\paragraph{How many of BigCodeBench's dimensions are real?}
BigCodeBench's Pearson-based $\text{ED} = 29$ is the highest in our corpus, but three validation experiments show that raw ED overstates true dimensionality.
Against a permutation null (100 column-shuffled replicates), only 8 PCs exceed the 95th percentile (PC1 at $12\times$ the null, declining to PC8 at $1.03\times$).
Split-half reliability (50 random model splits) confirms only PC1 ($\bar{r} = 0.86$) and marginally PC2 ($\bar{r} = 0.69$); PC3 onward is unstable ($\bar{r} < 0.5$).
For ranking preservation, 3 PCs suffice ($\tau = 0.997$; marginal gain $< 0.001$ per additional PC).

The convergent picture: ${\sim}8$--$9$ statistically significant dimensions, of which 2--3 are reliably reproducible.
Pearson-ED overstates by ${\sim}3$--$4\times$, consistent with binary-attenuation bias (Section~\ref{sec:method}).
Even so, BigCodeBench at 8 significant PCs remains $5$--$6\times$ more dimensional than BFCL (1 significant PC).
The implication: \textbf{independent measurement capacity is scarcer than raw ED values suggest across the entire evaluation ecosystem}.

\begin{figure}[htbp]
\centering
\includegraphics[width=\textwidth]{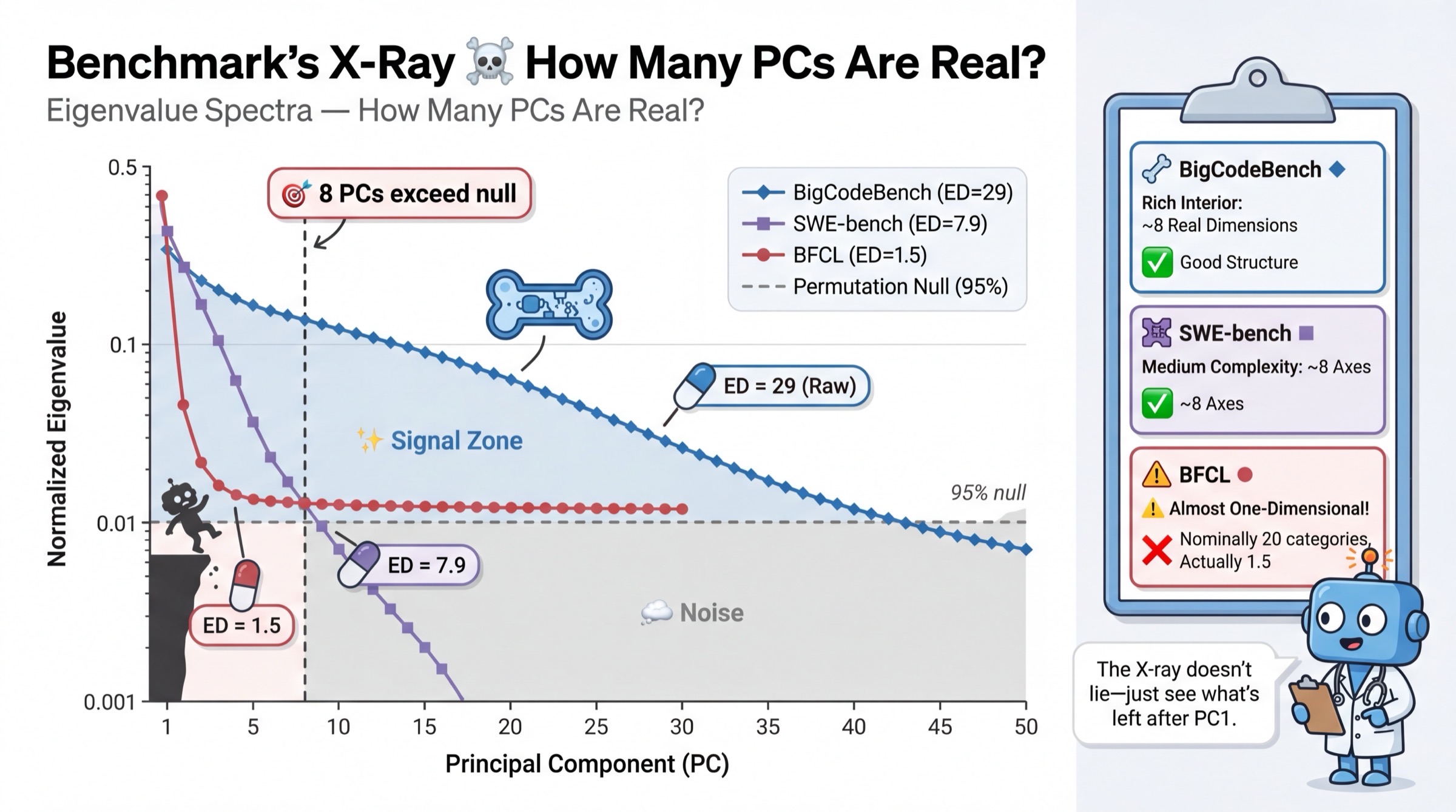}
\caption{Eigenvalue spectra reveal how many PCs are real.
BigCodeBench (blue) has 8 PCs exceeding the 95\% permutation null (shaded signal region); SWE-bench (purple) has ${\sim}4$--$6$; BFCL (red) collapses to one dominant axis.
Raw ED overstates true dimensionality by $3$--$4\times$.}
\label{fig:spectrum}
\end{figure}

\subsection{Saturation: When Is ED Reliable?}
\label{sec:saturation}

Saturation curves quantify how much discriminative capacity remains untapped (Figure~\ref{fig:causal}(b)).
Fitting $\text{ED}(n) = \text{ED}_\infty \cdot n / (n + n_{1/2})$, SWE-bench saturates rapidly ($\text{ED}_\infty = 8.4$, $n_{1/2} = 7$), while BigCodeBench saturates slowly ($\text{ED}_\infty = 35.4$, $n_{1/2} = 35$).
A saturated ED is a reliable benchmark property; an unsaturated ED is a lower bound.

At matched matrix dimensions ($300 \times 50$), BigCodeBench still shows higher ED than SWE-bench (Section~\ref{sec:landscape}); the full-table gap reflects BigCodeBench's larger item pool revealing more latent structure.
SWE-bench and BigCodeBench have nearly identical pairwise Hamming distances ($0.42$ vs.\ $0.47$) but ED differs nearly $4\times$ ($7.9$ vs.\ $28.9$); pairwise distance measures the magnitude of difference, not the number of independent directions (see Appendix~\ref{sec:app-additional} for full distributions).

\begin{figure}[htbp]
\centering
\includegraphics[width=0.95\textwidth]{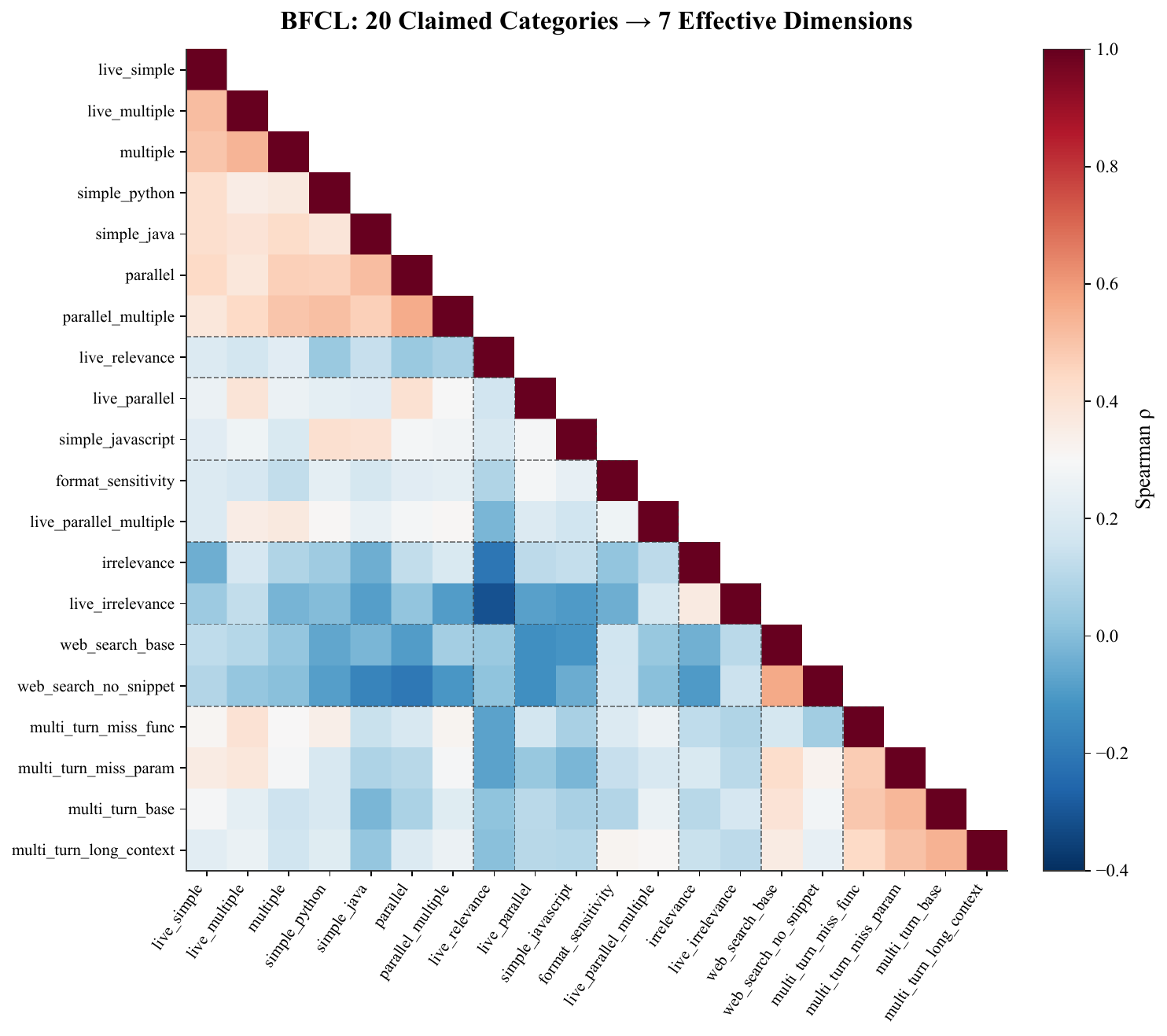}
\caption{BFCL's 20 claimed categories collapse into 7 effective groups (redundancy $= 2.8\times$).
Group~7: \texttt{simple\_python} $\approx$ \texttt{simple\_java} $\approx$ \texttt{simple\_javascript} ($r = 0.84$), three languages testing the same skill.}
\label{fig:bfcl}
\end{figure}

\begin{figure}[htbp]
\centering
\includegraphics[width=0.95\textwidth]{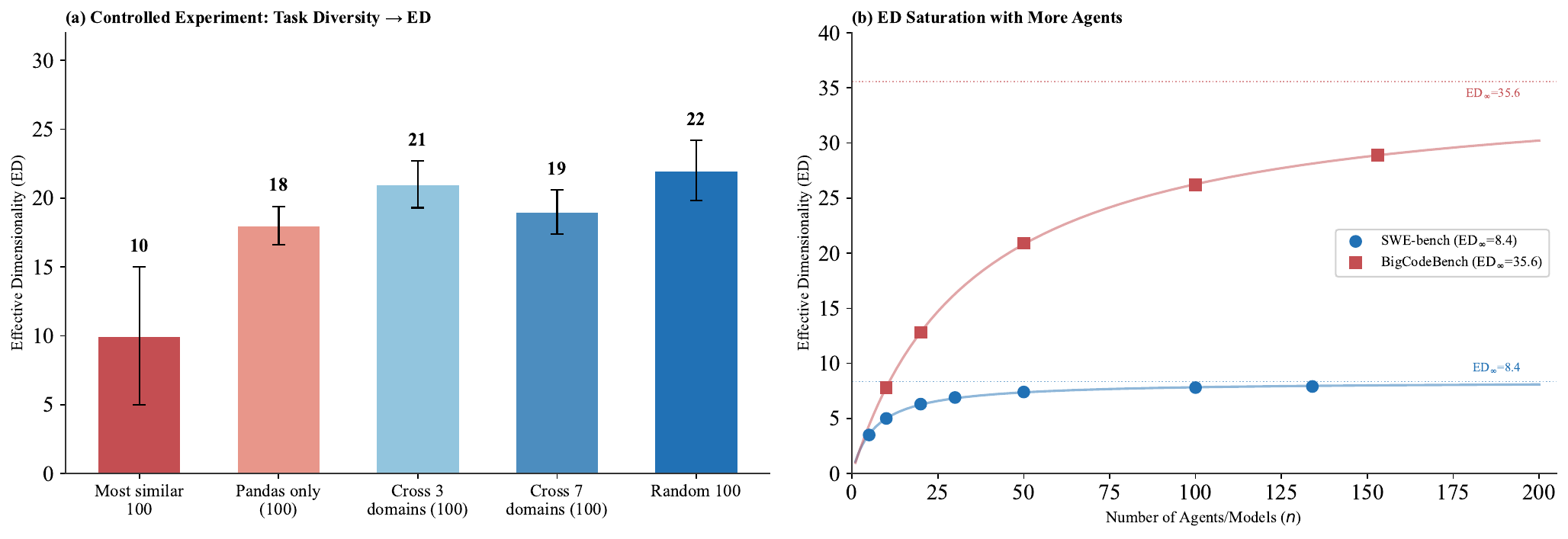}
\caption{(a)~Controlled experiment: structural heterogeneity, not topic diversity, drives ED.
(b)~Saturation curves: SWE-bench saturates at $\text{ED}_\infty \approx 8$; BigCodeBench saturates slowly ($\text{ED}_\infty \approx 35$).}
\label{fig:causal}
\end{figure}

%% file: sections/aging.tex
\section{Performance-Conditional Compression and Temporal Aging}
\label{sec:aging}

ED can serve as a monitoring tool: by tracking ED across performance strata or over time, maintainers can detect when a benchmark is losing discriminative breadth.
We distinguish two analyses below: performance-conditional compression (a cross-sectional analysis of the Open LLM Leaderboard, not a time series) and temporal aging (a true submission-time analysis of SWE-bench).

\paragraph{Performance-conditional compression: the Open LLM Leaderboard.}
On the full 4{,}576-model Open LLM Leaderboard, sorting models by composite score and computing ED in a sliding window of 500 models (step~200), ED declines from $2.65$ (weakest quintile) to $1.40$ (strong-model plateau), then partially recovers to $2.52$ among the strongest models (Mann--Kendall $\tau = -0.42$, $p = 0.007$).
To rule out ceiling effects, we standardize each benchmark to unit variance within each window before computing ED, thereby removing all variance-magnitude effects.
The ED decline not only persists but strengthens (Mann--Kendall $\tau = -0.72$, $p < 0.001$), confirming that the leaderboard's decline reflects a genuine change in the correlation structure---stronger models' benchmark scores become more positively correlated---not merely a reduction in score spread.

The recovery at the frontier is statistically robust: a bootstrap comparison (1{,}000 iterations, sampling 300 models from each group) yields $\text{ED}_{\text{plateau}} = 1.60 \pm 0.07$ vs.\ $\text{ED}_{\text{frontier}} = 2.68 \pm 0.12$, a difference of $1.08$ with 95\% CI $[0.81, 1.35]$ (Cohen's $d = 7.8$, $P(\text{frontier} > \text{plateau}) > 0.999$).
This diversity likely reflects heterogeneous training philosophies (reasoning-focused vs.\ instruction-following vs.\ knowledge-heavy) that produce more varied capability profiles than the incremental improvements characterizing the upper-middle tier.
If this pattern generalizes, benchmark maintainers need not retire a benchmark the moment ED declines; they should instead monitor whether frontier diversity is recovering.

\begin{figure}[htbp]
\centering
\includegraphics[width=0.85\textwidth]{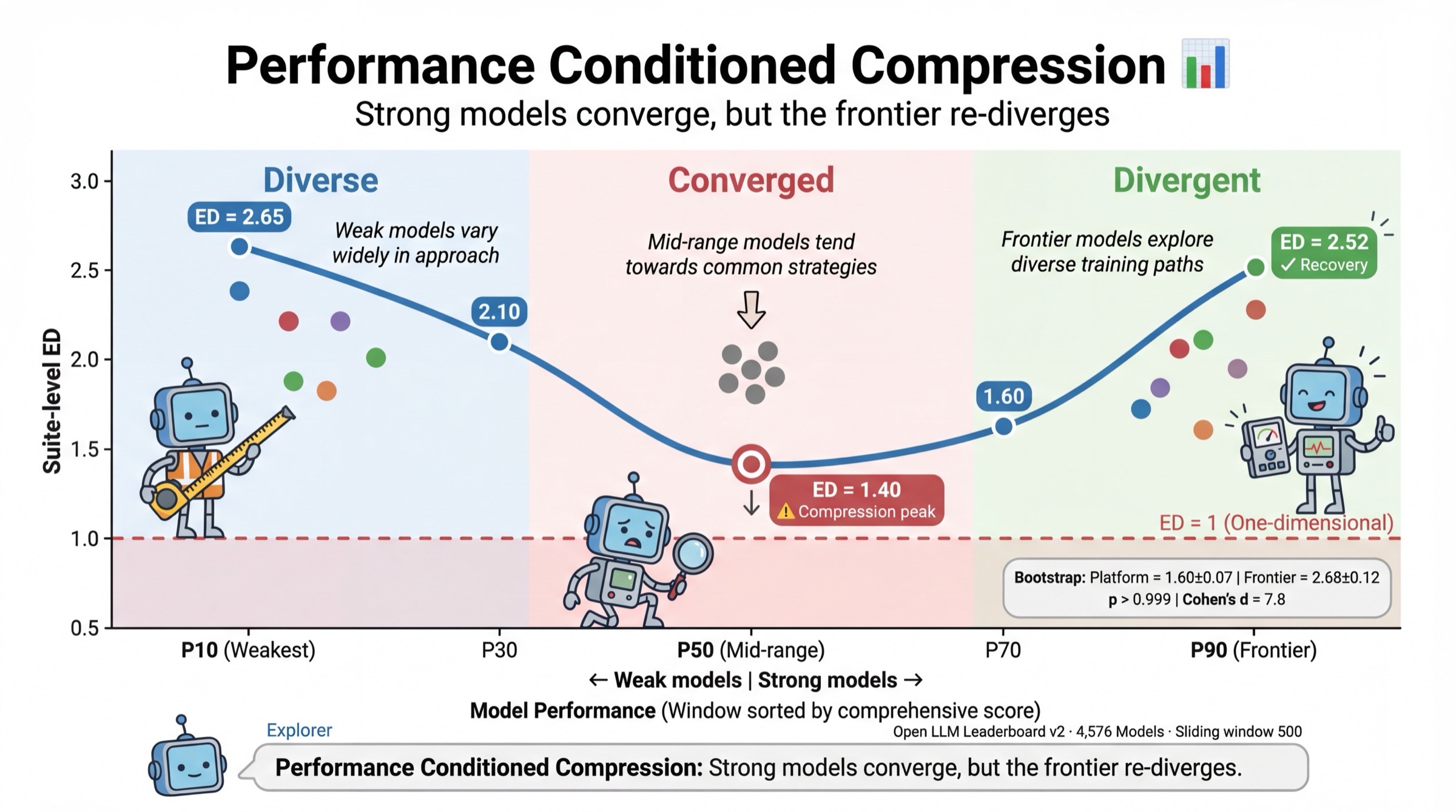}
\caption{Performance-conditional compression on the Open LLM Leaderboard (4{,}576 models, sorted by composite score, sliding window of 500).
ED declines from 2.65 (weakest quintile) to 1.40 (upper-middle plateau) then partially recovers to 2.52 among frontier models.
Note: the x-axis is a performance-ranked window, not a time axis.
The frontier recovery ($p > 0.999$, Cohen's $d = 7.8$) likely reflects heterogeneous training philosophies among top-performing models.}
\label{fig:lifecycle}
\end{figure}

\paragraph{SWE-bench: disentangling ceiling effects from structural change.}
SWE-bench Verified provides a per-instance time series: 134 agents submitted during a period of rapid capability growth.
ED declines 28\% as agents improve from 25\% to 66\% mean solve rate (Mann--Kendall $\tau = -0.56$, $p = 0.02$).
A controlled analysis cleanly separates genuine from artifactual components of this decline.
We identify 70 tasks (of 500) whose variance changes by $<$20\% between the early and late model cohorts.
On this fixed-variance subset, the ED trend disappears entirely (Mann--Kendall $\tau = 0.05$, $p = 1.0$): ED fluctuates between $13.6$ and $18.2$ with no directional trend.
Practitioners should therefore interpret per-instance ED trends with caution when pass rates approach~1; suite-level ED trends, computed on aggregate scores with sufficient variance, are more reliable.

BigCodeBench shows a related pattern: ordering models by aggregate performance, ED declines monotonically from $4.9$ to $2.7$ (Mann--Kendall $\tau = -1.0$, $p < 0.001$).
Adding architectural diversity can partially reverse this: inserting one randomly sampled early-era agent into the late SWE-bench cohort increases ED in 86\% of bootstrap trials, demonstrating that aging reflects the current homogeneity of the model population, not an irreversible property.

\paragraph{Summary: what is real and what is artifact.}
\begin{center}
\small
\begin{tabular}{lcc}
\toprule
\textbf{Analysis} & \textbf{Genuine?} & \textbf{Evidence} \\
\midrule
Open LLM suite-level decline   & Genuine         & Survives variance standardization ($\tau = -0.72$) \\
Frontier recovery               & Genuine         & Bootstrap $p > 0.999$, Cohen's $d = 7.8$ \\
SWE-bench per-instance decline & Partly artifact & Disappears on fixed-variance tasks \\
\bottomrule
\end{tabular}
\end{center}

\begin{figure}[htbp]
\centering
\includegraphics[width=\textwidth]{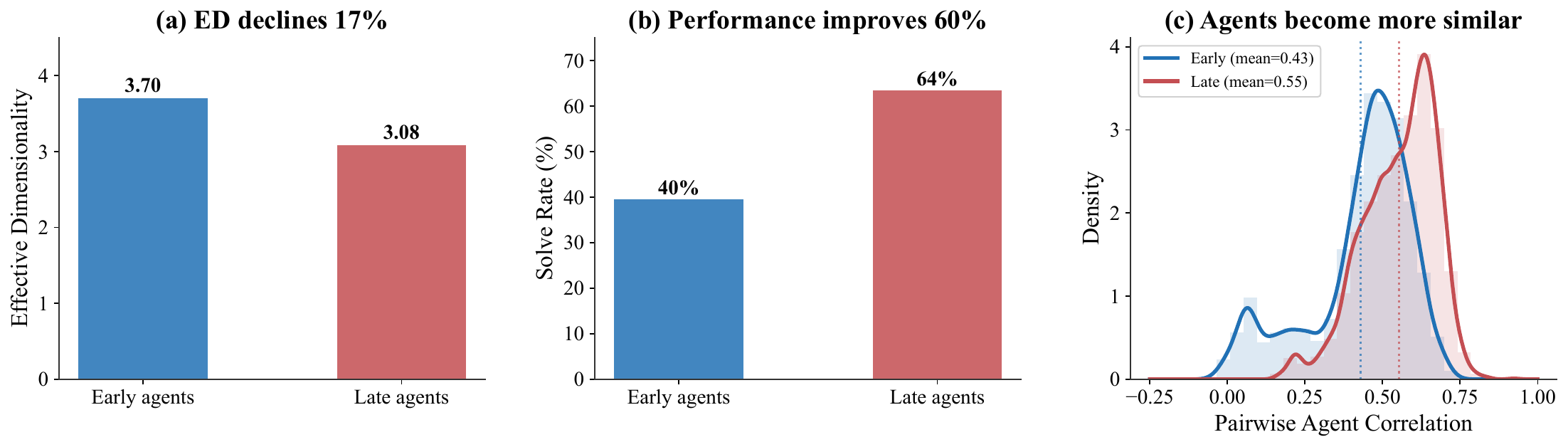}
\caption{Benchmark aging.
(a)~SWE-bench ED declines 17\% from early to late agents.
(b)~Solve rate improves 60\% (40\% to 64\%).
(c)~Pairwise agent correlations shift rightward (mean $\rho = 0.43 \to 0.55$): late agents are 29\% more similar to each other.
A fixed-variance control shows the SWE-bench decline is partly a ceiling effect; the leaderboard-level decline survives variance standardization.}
\label{fig:aging}
\end{figure}

%% file: sections/semantics.tex
\section{Semantic Validation of Principal Components}
\label{sec:semantics}

SWE-bench Verified ($\text{ED} \approx 8$) and BigCodeBench ($\text{ED} \approx 29$) sit at opposite ends of the dimensionality spectrum.
We validate that the structure detected by ED corresponds to meaningful capability dimensions using external metadata---not the PC scores themselves---to avoid circular validation.

\paragraph{SWE-bench Verified.}
PC1 (50.5\% of variance) is aggregate ability ($r = 0.999$ with solve rate).
PC2 (7.3\%) captures a generational shift: classifying agents as pre- vs.\ post-2025 using PC2 alone yields $\text{AUC} = 0.97$, reflecting a qualitative change in repository-navigation strategies not reducible to overall ability.
PC3 (${\sim}2\%$) separates open-weight from proprietary-API agents.
PC4 (${\sim}1\%$) resolves domain specificity (Django vs.\ Sphinx).

\paragraph{BigCodeBench.}
PC1 (16.5\%) is general programming ability---but it accounts for far less variance than SWE-bench's PC1, leaving 83.5\% across many axes.
PC2 (3.8\%) separates code-specialized models (AutoCoder, Magicoder) from general-purpose models (DeepSeek-R1, chat models); classifying by model-card purpose yields $\text{AUC} = 0.94$.
PC3 (2.2\%) separates instruction-tuned from base variants.

Beyond PC3, components carry interpretable but increasingly fragile signals.
PC4 (1.5\%) separates small code-specialists (DeepSeek-Coder-6.7B) from general large models (Mistral-Small, Claude-3-Sonnet), with task poles distinguishing os/collections-heavy tasks from sklearn/scipy-heavy tasks, a systems-programming-vs.-data-science axis.
PC6 (1.4\%) separates the Llama-3-70B family from GPT-4-class models.
PC9 (1.2\%) cleanly splits web-scraping tasks (bs4, requests) from data-analysis tasks (pandas, sklearn), with zero library overlap between poles.
However, split-half reliability (Section~\ref{sec:what-drives-ed}) shows that only PC1--2 are reproducible across model subsets; PC3 onward depends on which models are included.

The semantic identity of principal components is population-conditional: as new model families emerge, the latent structure shifts.
This is a strength of ED as a living diagnostic rather than a frozen taxonomy.

\begin{figure}[htbp]
\centering
\includegraphics[width=\textwidth]{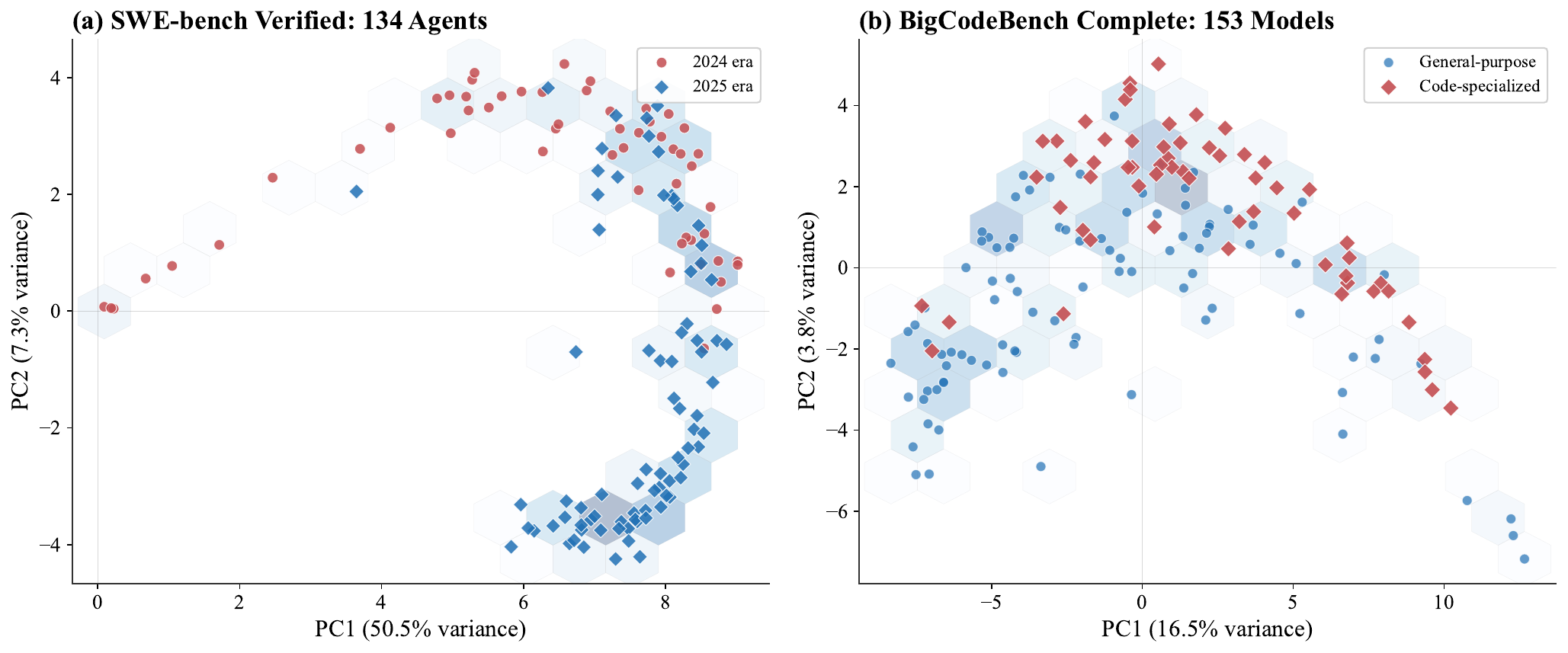}
\caption{Semantic validation.
(a)~SWE-bench: PC2 separates 2024-era agents (red) from 2025-era (blue), $\text{AUC} = 0.97$.
(b)~BigCodeBench: PC2 separates code-specialized (red) from general-purpose models (blue), $\text{AUC} = 0.94$.}
\label{fig:semantics}
\end{figure}

%% file: sections/criteria.tex
\section{A Diagnostic Workflow for Benchmark Maintainers}
\label{sec:guidelines}

\paragraph{How to interpret ED (and how not to).}
ED is a population-conditional spectral summary, not a literal count of independent skills.
\textbf{Do:} use ED to compare benchmarks' relative measurement breadth within comparable model populations; to flag redundancy when suite-level ED is close to~1; to estimate evaluation cost (higher ED $\to$ harder to compress).
\textbf{Do not:} interpret ED as an exact factor count; it provides a well-calibrated upper bound (tetrachoric correction yields tighter estimates; Appendix~\ref{sec:app-additional}); compare raw ED values across benchmarks with very different model populations or matrix shapes without matched-dimension controls; or use ED as a standalone benchmark quality score, since low ED is appropriate for focused evaluations (e.g., MATH), and high ED does not guarantee practical usefulness.

\subsection{Four-Step Diagnostic Workflow}
\label{sec:workflow}

\begin{figure}[htbp]
\centering
\includegraphics[width=\textwidth]{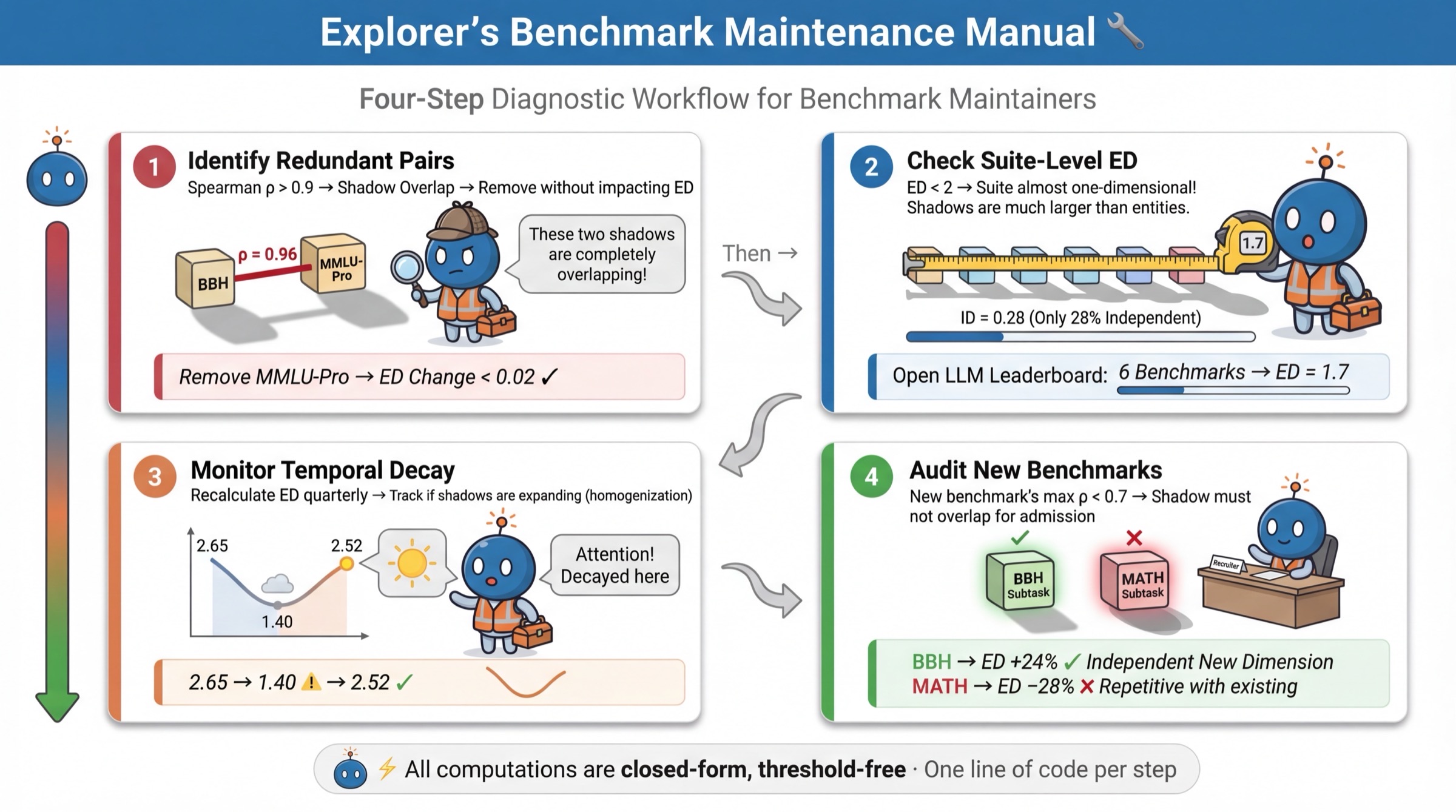}
\caption{Four-step diagnostic workflow for benchmark maintainers.
(1)~Flag redundant pairs ($\rho > 0.9$); (2)~check suite-level ED; (3)~monitor temporal decline; (4)~vet new additions ($\max \rho < 0.7$).
All computations are closed-form and require only a score matrix.}
\label{fig:workflow}
\end{figure}

Given a benchmark suite, we recommend the following sequence.
Each step is grounded in empirical findings from the preceding sections; all computations are closed-form and require only the score matrix.

\begin{enumerate}[leftmargin=1.8em, itemsep=4pt]
    \item \textbf{Flag redundant pairs} ($\boldsymbol{\rho > 0.9}$).
    Compute the pairwise Spearman correlation matrix and flag pairs exceeding $0.9$ as redundancy candidates.
    BBH and MMLU-Pro ($\rho = 0.96$; Section~\ref{subsec:redundancy}) are the canonical example: removing MMLU-Pro changes ED by $< 0.02$.
    Strong negative $\rho$ signals complementarity, not redundancy, and should be retained.

    \item \textbf{Check suite-level ED.}
    Compute ED on the suite's score matrix.
    If $\text{ED} < 2$, the suite is effectively one-dimensional and adding benchmarks is the priority.
    The Open LLM Leaderboard v2 has $\text{ED} = 1.7$ with information density $\text{ID} = 0.28$ (Section~\ref{subsec:redundancy}).

    \item \textbf{Monitor temporal decline.}
    Recompute ED quarterly as new models are evaluated.
    A sustained decline may signal population homogenization (Section~\ref{sec:aging}).
    A benchmark whose ED declines monotonically with no frontier recovery is more urgently in need of replacement than one exhibiting a U-shaped pattern.

    \item \textbf{Vet new additions} ($\boldsymbol{\max \rho < 0.7}$).
    When adding a new benchmark, verify that its maximum positive $\rho$ with existing benchmarks is below $0.7$; if it exceeds $0.9$, it adds little independent information.
\end{enumerate}

\noindent
\textbf{When strong claims are needed} (e.g., retiring a benchmark, publishing an ED-based comparison), complement the screening steps above with:
\begin{itemize}[leftmargin=1.5em, itemsep=2pt]
    \item a permutation-null test to identify which PCs exceed random structure;
    \item a saturation curve to check whether ED has converged at the current model count;
    \item matched-dimension subsampling if comparing benchmarks with very different matrix shapes.
\end{itemize}
ED computation is closed-form; interpreting absolute dimensionality may require these additional checks.

\subsection{Screening Heuristics}
\label{sec:heuristics}

The following thresholds are empirically derived from the 22-benchmark landscape under task-centering (Table~\ref{tab:reference}) and should be applied within a consistent preprocessing pipeline.

\paragraph{Heuristic 1: Match ED to purpose.}
ED measures discriminative breadth, not evaluation quality.
MATH ($\text{ED} = 1.3$) is an excellent focused test of mathematical reasoning; its unidimensionality is a design strength, appropriate for assessing specialized capabilities.
High ED is desirable when the goal is a comprehensive suite that captures multiple independent capabilities:
\begin{itemize}[leftmargin=1.25em,itemsep=2pt]
\item $\text{ED} < 3$ (tetrachoric-corrected: ${\sim}1$--$2$): single-axis ranking, appropriate for focused evaluations, but insufficient as a standalone comprehensive suite.
\item $\text{ED} \in [5, 10]$ (corrected: ${\sim}3$--$5$): multi-dimensional discrimination, the regime of MMLU-Pro ($6.8$) and LiveBench ($5.9$).
\item $\text{ED} > 10$ (corrected: $>5$): rich measurement, but must be confirmed by a saturation curve (Section~\ref{sec:saturation}) and semantic validation (Section~\ref{sec:semantics}).
\end{itemize}
\noindent
(All thresholds are based on Pearson-ED, which is an upper bound; tetrachoric-corrected values are approximately half as large.)

\paragraph{Heuristic 2: Low positive redundancy with existing benchmarks.}
A new benchmark contributes independently only if $\max \rho < 0.7$; above this threshold, marginal information is small.
$\max \rho > 0.9$ is outright redundancy---BBH and MMLU-Pro are the canonical example---and one can be retired without information loss.

\paragraph{Heuristic 3: Conditional negative correlations signal complementarity.}
Within a fixed model-size class, benchmarks negatively correlated ($\rho < -0.3$) with an existing evaluation are especially informative, exposing capability dimensions that current optimization targets do not reward.
These negative correlations are conditional on model scale (Section~\ref{sec:redundancy-tradeoffs}) and may not generalize to all populations.

\subsection{Role-Specific Guidance}
\label{sec:role-guidance}

\paragraph{For benchmark designers.}
Compute ED from the eigenvalue spectrum, not from category counts, since the number of stated task types is an unreliable proxy for measurement dimensionality (Appendix~\ref{sec:app-category-count}).
Use inter-category correlation matrices and hierarchical clustering (Section~\ref{sec:bfcl-case}) to verify which categories are genuinely independent before advertising breadth.

\paragraph{For benchmark users.}
Check the pairwise $\rho$ matrix before reporting scores on multiple benchmarks: ``excels on six benchmarks'' is uninformative when four of the six are echoes of a single factor.
For long-tail deployment contexts, per-domain analysis is essential, as global rankings can mask substantial variation across task subsets.

\paragraph{For model developers.}
Within a fixed model-size class, improving one benchmark is often associated with lower scores on another: the MATH--MuSR conditional correlation ($\rho = -0.64$) is the sharpest example (Section~\ref{sec:redundancy-tradeoffs}).
Monitor both the target benchmark and its negatively correlated counterparts.

\subsection{Advanced: Task Selection with ED-Greedy}
\label{sec:ed-greedy}

Beyond diagnosis, ED suggests a constructive application: selecting tasks to maximize measurement diversity.
\textsc{ED-Greedy} (Algorithm~\ref{alg:ed-greedy}) starts from an empty set and iteratively adds the task producing the largest marginal increase in ED, with complexity $O(k \cdot T \cdot N^2)$.

\begin{algorithm}[htbp]
\caption{\textsc{ED-Greedy}: Task Selection for Maximum Effective Dimensionality}
\label{alg:ed-greedy}
\begin{algorithmic}[1]
\REQUIRE Pass/fail matrix $M \in \{0,1\}^{T \times N}$, budget $k$
\ENSURE Selected task set $S^*$ with $|S^*| = k$
\STATE $S \leftarrow \emptyset$
\FOR{$i = 1$ to $k$}
  \STATE $t^* \leftarrow \arg\max_{t \in [T] \setminus S}\; \text{ED}(M_{S \cup \{t\}})$
  \STATE $S \leftarrow S \cup \{t^*\}$
\ENDFOR
\RETURN $S$
\end{algorithmic}
\end{algorithm}

At $k = 50$ tasks on BigCodeBench, ED-Greedy achieves $\text{ED} = 37.8$ versus $4.2$ for random selection.
At $k = 50$, it also achieves competitive ranking fidelity ($\tau = 0.85$, comparable to IRT-based selection at $\tau = 0.84$).
The IRT-inspired baseline is conceptually similar to the tinyBenchmarks approach of \citet{polo2024tinybenchmarks}; ED-Greedy optimizes a complementary objective (measurement diversity) that tinyBenchmarks does not target.

\paragraph{ED tracks compression cost.}
Across eight benchmarks with $T \geq 100$ per-instance items, higher-ED benchmarks generally need more tasks to achieve $\tau \geq 0.95$ with the full-task ranking (Spearman $\rho = 0.64$, $n = 8$).
BFCL ($\text{ED} = 1.5$) achieves $\tau \geq 0.95$ with just 8\% of its tasks; BigCodeBench ($\text{ED} = 29$) requires 56\%.
Two exceptions---MMLU-Pro and BBH---have high ED ($20$--$27$) yet compress relatively easily, because their large task pools include substantial internal redundancy.
The practical lesson: ED flags which benchmarks cannot be aggressively compressed (high ED + moderate $T$), but for benchmarks with very large $T$, internal item redundancy matters as much as spectral diversity.

\paragraph{Prospective generalization.}
Splitting models into a design cohort (bottom 60\%) and future cohort (top 40\%), ED-Greedy generalizes well: on SWE-bench, it beats random on the future cohort ($\tau = 0.81$ vs.\ $0.77$ at $k = 100$), with essentially no overfitting.
The high-variance baseline catastrophically overfits in all conditions.
On low-ED benchmarks (BFCL, $\text{ED} = 1.5$), random selection is already near-optimal.
The practical rule: ED-Greedy adds value when ED $> 3$ and the budget covers $\geq$10\% of tasks; otherwise random suffices.

%% file: sections/related.tex
\section{Related Work}
\label{sec:related}

Despite the rapid proliferation of AI benchmarks (over 50 major evaluation suites exist as of 2025), there is no established methodology for measuring whether a benchmark provides independent information or how much of its claimed breadth is redundant.
We organize prior work by the claim it advances, grouping contributions into three threads: low-dimensional structure and redundancy, conditional negative correlations, and benchmark quality criticism.
Each thread identifies part of the problem; none provides the systematic, computable framework that ED delivers.

\paragraph{Low-dimensional structure and redundancy.}

Several recent studies have independently observed that LLM evaluation suites contain far fewer independent axes than their task taxonomies suggest.
\citet{kearns2026structure} apply PCA to the Open LLM Leaderboard and identify a dominant general factor that explains the majority of cross-model variance, echoing the $g$-factor tradition in psychometrics.
\citet{federiakin2025psychometric} apply Confirmatory Factor Analysis to IFEval's 25 instruction-following types and find that they load onto far fewer latent factors than the taxonomy advertises, with several type pairs exhibiting negative inter-factor correlations.
\citet{polo2024tinybenchmarks} and \citet{hagos2024sloth} exploit this low-rank structure for practical purposes, showing that small item subsets suffice to predict full-benchmark scores with high fidelity, a strategy that is viable precisely because the effective rank is low.
\citet{zhang2025efficienteval} formalize efficient evaluation for multimodal LLMs as an active-testing problem, providing theoretical guarantees on ranking preservation under subsampling.

The psychometrics literature offers dedicated dimensionality-assessment tools.
DETECT~\citep{zhang2004conditional} uses conditional covariances to identify clusters of items sharing a latent dimension; DIMTEST~\citep{stout1987nonparametric} provides a hypothesis test for unidimensionality based on item-pair covariances after matching on a nuisance composite.
Both methods are well-suited to educational testing, where item counts are moderate and examinee populations are large.
DETECT and DIMTEST offer greater statistical rigor than the participation ratio; they are designed for formal hypothesis testing with well-characterized Type~I error.
We chose the participation ratio for a different design goal: a diagnostic that trades some statistical power for tractability and deployability.
Specifically:
(i)~it is closed-form and requires no iterative estimation, making it tractable for matrices with $>$10{,}000 items;
(ii)~it yields a continuous scalar rather than a binary accept/reject decision, enabling cross-benchmark comparisons and temporal tracking;
(iii)~it does not require specifying a partition of items into clusters or a nuisance dimension, avoiding assumptions that may not hold when the latent structure is unknown.
For benchmarks where ED indicates low dimensionality ($\text{ED} < 3$), DETECT and DIMTEST provide granular item-cluster analysis as a natural complement.
Related spectral surrogates include the stable rank~\citep{rudelson2007sampling} and Shannon effective rank~\citep{roy2007effective}; ED (the R\'{e}nyi-2 effective rank) is more conservative ($\leq \exp(H_1)$), which is appropriate for binary benchmark matrices where spectra are strongly spiked.

\paragraph{Population-conditional spectral analysis of benchmarks.}
The participation ratio originates in condensed-matter physics~\citep{thouless1974,bell1970} and computational neuroscience~\citep{gao2017,stringer2019}, where the matrix under analysis represents a fixed physical or biological quantity.
Our conceptual departure is to treat benchmark dimensionality as a \emph{population-conditional} diagnostic: not ``how many categories does this benchmark have?'' (a static question) but ``how many independent signals does it currently provide?'' (a dynamic question).
This interactive perspective---benchmarks as dynamic measurement instruments rather than fixed rulers---has no precedent in AI evaluation.

Our work differs from prior PCA- and IRT-based studies in three respects:
(i)~per-instance granularity, which we show inflates ED by $2$--$4\times$ at category level;
(ii)~cross-benchmark scale (22 benchmarks, 8 domains, 8{,}400+ models);
(iii)~systematic quantification of conditional negative correlations across all benchmark pairs, not isolated case studies.

\paragraph{Conditional negative correlations across benchmarks.}

A growing body of empirical work documents specific capability tensions in LLMs.
\citet{fu2025scaling} show that training models with extended chain-of-thought traces improves formal reasoning benchmarks but degrades instruction-following accuracy, a phenomenon they term ``scaling reasoning, losing control.''
\citet{li2025reasoning} document a parallel tension between reasoning ability and safety alignment: models fine-tuned for stronger reasoning exhibit increased rates of harmful output generation, suggesting that reasoning optimization and safety optimization compete for model capacity.
\citet{federiakin2025psychometric} report negative inter-factor correlations within IFEval, hinting that even within a single benchmark, negative correlations can arise.
\citet{zhang2024structural} approach the problem from a social-choice perspective, showing that aggregation methods face a formal diversity--stability tension when the underlying capability axes are negatively correlated.

These studies each isolate one or two negative-correlation pairs under controlled conditions.
Our contribution is systematic scope: we quantify all 15 pairwise correlations on the Open LLM Leaderboard v2 with bootstrap confidence intervals, partial correlations controlling for model size, and size-stratified analysis (Section~\ref{sec:redundancy-tradeoffs}).
The result---40\% of pairs significantly anti-correlated---reveals that the isolated negative correlations reported in prior work are instances of a pervasive conditional phenomenon (Simpson's paradox), not anomalies.

\paragraph{Benchmark quality and criticism.}

A parallel literature critiques AI benchmarks on conceptual and methodological grounds.
\citet{raji2021ai} argue that benchmarks are routinely used as proxies for capabilities they were never designed to measure, inflating claims of progress.
\citet{eriksson2025benchmarks} provide a taxonomy of benchmark failure modes (construct under-representation, label noise, metric misalignment), offering a conceptual framework for diagnosing flawed evaluations.
\citet{reuel2024betterbench} propose best practices for benchmark construction grounded in measurement theory, emphasizing reliability, validity, and documentation standards.

These critiques are predominantly qualitative: they identify what can go wrong and recommend what should be done, but they do not provide computable diagnostics for specific benchmarks.
Our work supplies the quantitative counterpart.
Where prior work says ``benchmarks may be redundant,'' we show that BBH and MMLU-Pro correlate at $\rho = 0.96$; where prior work warns that composite scores may be misleading, we identify exactly which pairs are negatively correlated conditional on model size, and by how much.
Effective Dimensionality, pairwise redundancy analysis, and conditional-correlation detection are tools that can be applied to any benchmark with a public score matrix, converting qualitative intuitions into falsifiable measurements.

%% file: sections/discussion.tex
\section{Discussion and Limitations}
\label{sec:discussion}

\subsection{Limitations}
\label{sec:limitations}

\textbf{Binary SVD.}
Pearson-based ED applied to binary matrices overestimates true dimensionality by $1.5$--$8\times$ (Section~\ref{sec:method}); tetrachoric correction halves the values while preserving the rank ordering.
All reported ED values are upper bounds.
\textbf{Population dependence.}
ED is a property of the benchmark--population interaction: restricting BigCodeBench from 153 to 10 models reduces ED from 29 to 7.
Saturation curves and matched-dimension controls (Section~\ref{sec:landscape}) confirm that relative comparisons are robust ($\rho = 1.0$ at $300 \times 50$); absolute values are population-conditional.
\textbf{Correlation $\neq$ identity.}
Two items with highly correlated pass/fail vectors may test the same capability or distinct capabilities that current models happen to co-vary on.
All conclusions about redundancy are conditioned on the current model ecosystem and should be recomputed as the ecosystem evolves.
\textbf{Conditional correlations.}
The negative correlations (e.g., MATH--MuSR $\rho = -0.64$) are specific to the 188-model subset and disappear in the full population (Section~\ref{sec:redundancy-tradeoffs}).
\textbf{Shape constraints.}
ED cannot exceed $\min(T, N)$; benchmarks with $\leq 6$ task groups are flagged in Table~\ref{tab:reference}.

\subsection{Reliability Tiers}
\label{sec:reliability}

Recomputing all statistics on seven subpopulations of the Open LLM Leaderboard:
\textbf{High reliability:} BBH$\approx$MMLU-Pro redundancy ($\rho \in [0.88, 0.97]$), granularity artifact, ranking fragility (38\% champion change).
\textbf{Medium:} MATH--MuSR negative correlation is conditional (positive in every subpopulation unconditionally).
\textbf{Interpret with caution:} Absolute ED values ($1.49$--$2.20$ across subpopulations for the 6-benchmark suite).

\subsection{Broader Discussion}
\label{sec:broader}

\paragraph{Maintainer decision studies.}
For pruning: ED correctly identifies IFEval as irreplaceable ($|\Delta\text{ED}| = 0.30$, $\tau = 0.83$ after removal).
For addition: adding a BBH subtask increases suite ED by $+24\%$, while adding a MATH subtask decreases it by $-28\%$; disaggregating an existing benchmark can be more valuable than adding a new broad-coverage suite.

\paragraph{Toward a new paradigm for benchmark design.}
Current benchmark design is intuition-driven: designers list categories, collect items, and publish, with no standard step to check whether the resulting benchmark measures multiple independent capabilities.
ED introduces empirical measurement of benchmark quality as a standard practice, analogous to test coverage in software engineering.
Designers can compute ED on a pilot set to identify redundant task clusters before publication; maintainers can monitor ED over time to detect homogenization (Section~\ref{sec:aging}); leaderboard operators can use pairwise $\rho$ and suite-level ED to decide which benchmarks to include.
The ED-Greedy algorithm (Section~\ref{sec:ed-greedy}) operationalizes this as a task-selection procedure.

\paragraph{Future directions.}
The conditional negative correlations call for causal verification via controlled training experiments.
Our empirical evidence ($\rho = 0.64$ across 8 benchmarks) that high-ED benchmarks produce more robust rankings awaits theoretical formalization.
Deploying ED as a real-time health diagnostic that automatically flags sustained declines would operationalize the aging findings at scale.

\subsection{Conclusion}
\label{sec:conclusion}

AI evaluation suites often report many scores without checking whether those scores carry independent information.
We introduced effective dimensionality as a fast, population-conditional upper-bound diagnostic of measurement breadth, and applied it to 22 benchmarks across 8 domains.
The central finding: benchmark suites contain far less independent measurement breadth than their score sheets suggest.
The Open LLM Leaderboard behaves like roughly two effective measurement axes ($\text{ED} = 1.7$), BFCL's 20 nominal categories collapse under spectral analysis, and even BigCodeBench ($\text{ED} = 29$) has only ${\sim}8$ statistically significant dimensions when tested against a permutation null.

The key design principle is structural heterogeneity---tasks that discriminate models in genuinely different ways---and ED provides a fast screening metric to verify this at construction time.
For practitioners: a four-step workflow with validation checks (Section~\ref{sec:guidelines}), a score matrix, and a few lines of code.
The 22-benchmark atlas is a snapshot; the screening methodology is permanent.

%% file: appendix/data_details.tex
\section{Data Sources and Preprocessing}
\label{sec:app-data}

This appendix describes the provenance and preprocessing of each data source used in the paper.

\paragraph{SWE-bench experiments repository.}
Pass/fail matrices for SWE-bench Test, Lite, Verified, Multilingual, and Multimodal were extracted from the official SWE-bench experiments repository on GitHub.\footnote{\url{https://github.com/swe-bench/experiments}}
Each row corresponds to a GitHub issue instance and each column to an agent submission.
Entries are natively binary (resolved or not), so no binarization was required.
We retained all agent submissions present in the repository as of our data collection date, yielding matrices ranging from 12 agents (Multimodal) to 134 agents (Verified).

\paragraph{HuggingFace datasets.}
Score matrices for LiveBench, LiveCodeBench, and GAIA were obtained from their respective HuggingFace dataset pages.
LiveBench and LiveCodeBench provide per-question binary correctness judgments; GAIA provides per-level pass rates.
For GAIA, which reports continuous accuracy scores at three difficulty levels, we binarized at a threshold of 0.5: scores above 0.5 are coded as pass, scores at or below 0.5 as fail.

\paragraph{BFCL GitHub repository.}
The Berkeley Function Calling Leaderboard (BFCL) score matrix was extracted from its official GitHub repository.\footnote{\url{https://github.com/ShishirPatil/gorilla/tree/main/berkeley-function-call-leaderboard}}
The raw data provides per-instance pass/fail judgments across 20 function-calling categories and 109 models.
We used the per-instance granularity for our primary analysis and additionally computed category-level aggregates for the redundancy analysis in Section~\ref{sec:bfcl-case}.

\paragraph{BigCodeBench.}
Per-task pass/fail matrices for BigCodeBench (complete and instruct variants, standard and hard splits) were obtained from the BigCodeBench HuggingFace collection.\footnote{\url{https://huggingface.co/collections/bigcode/bigcodebench}}
Each entry indicates whether a model's generated code passes the task's unit test suite.
The data is natively binary.
Model counts range from 126 (instruct) to 199 (hard-complete).

\paragraph{Open LLM Leaderboard v2.}
Scores for 4,576 models across six benchmarks (BBH, GPQA, IFEval, MATH, MMLU-Pro, MuSR) were downloaded from the HuggingFace Open LLM Leaderboard v2.\footnote{\url{https://huggingface.co/spaces/open-llm-leaderboard/open_llm_leaderboard}}
At per-benchmark granularity, each model receives six continuous scores (normalized accuracy or exact-match rates).
At per-subtask granularity, scores are available for 43 subtasks across the six benchmarks for a subset of 188 models.
Per-subtask scores were binarized at 0.5.

\paragraph{MMLU-Pro.}
Two MMLU-Pro matrices are used in this paper.
(1)~The 47-model matrix ($12{,}257 \times 47$, $\text{ED} = 6.8$) from the TIGER-Lab GitHub repository\footnote{\url{https://github.com/TIGER-AI-Lab/MMLU-Pro}} appears in Table~\ref{tab:reference}.
(2)~A larger 118-model matrix ($12{,}032 \times 118$, $\text{ED} = 20.8$) was assembled from per-question results in the Open LLM Leaderboard v2 detail datasets; this matrix is used only in the compression analysis (Section~\ref{sec:ed-greedy}).
The higher ED in the 118-model matrix is consistent with the saturation analysis: more models reveal more latent axes.
Both matrices are natively binary (correct or incorrect).

\paragraph{Standardization.}
Two preprocessing steps were applied uniformly across all datasets.
First, any continuous score was binarized at a threshold of 0.5.
This threshold was chosen as a natural midpoint; sensitivity analyses with thresholds of 0.3 and 0.7 produced qualitatively identical ED rankings (see supplementary materials).
Second, missing entries (NaN values arising from models not evaluated on all items) were imputed with the column mean---i.e., the mean pass rate of the model across all items for which it was evaluated.
This is a conservative imputation strategy that neither inflates nor deflates the variance contributed by any particular task, and is standard practice in PCA applications with incomplete data.

%% file: appendix/additional.tex
\section{Additional Experimental Results}
\label{sec:app-additional}

This appendix summarizes supplementary analyses that support the main text.
Full tables, figures, and per-benchmark breakdowns are provided in the supplementary materials.

\paragraph{Centering sensitivity.}
Table~\ref{tab:centering} reports ED under three centering methods for the three largest per-instance benchmarks.
Task-centering (our default) removes each task's mean pass rate across models; model-centering removes each model's mean pass rate across tasks; double-centering removes both.
The benchmark ordering is stable across all methods, supporting the robustness of the relative comparisons in the main text.

\begin{table}[htbp]
\centering
\caption{ED under three centering methods. The relative ordering (BigCodeBench $\gg$ SWE-bench $>$ BFCL) is stable across all methods.}
\label{tab:centering}
\small
\begin{tabular}{lccc}
\toprule
\textbf{Benchmark} & \textbf{Task-center} & \textbf{Model-center} & \textbf{Double-center} \\
\midrule
BigCodeBench & 28.9 & 4.4 & 62.3 \\
SWE-bench Verified & 7.9 & 3.8 & 30.8 \\
BFCL & 1.5 & 3.1 & 3.6 \\
\bottomrule
\end{tabular}
\end{table}

\paragraph{Continuous vs.\ binarized ED.}
For benchmarks with continuous scores, binarization at threshold 0.5 can affect absolute ED values.
On the Open LLM 6-benchmark suite, the continuous-score ED is $1.66$; binarized at $0.3$ it rises to $2.65$, and at $0.4$ to $2.96$.
Higher thresholds ($>0.5$) produce degenerate rows (all models pass or all fail) and are not usable.
The qualitative conclusion---severe redundancy, $\text{ED} \ll 6$---is identical under both continuous and binary analysis.
For the 43-subtask Open LLM matrix, continuous ED ($4.50$) and binarized ED ($4.14$--$7.08$ depending on threshold) also agree on the ``moderate dimensionality'' tier.
We recommend that future work on continuous-score benchmarks report continuous-score ED as the primary metric.

\paragraph{Diagnostic comparison: ED vs.\ alternatives.}
We compared three spectral diagnostics---ED, $-\text{PC1\%}$ (negated so that higher = more compressible), and Shannon effective rank---on their ability to predict benchmark compressibility (minimum fraction of tasks for $\tau \geq 0.95$) across eight benchmarks.
ED achieves the highest rank correlation ($\rho = 0.64$), compared to Shannon effective rank ($\rho = 0.57$) and $-\text{PC1\%}$ ($\rho = 0.55$).
On a binary decision task (``can this benchmark be compressed to 20\% of tasks?''), ED correctly classifies 7/8 benchmarks vs.\ 6/8 for PC1\%.
ED achieves the highest predictive accuracy across all comparisons, while requiring only a single closed-form computation with no threshold or iterative fitting.

\paragraph{Synthetic validation of ED.}
To assess ED's recovery accuracy, we generated binary pass/fail matrices from $k$-factor latent models (2PL IRT with $k$ ability dimensions, $T = 500$ tasks, $N = 100$ models) for $k \in \{1, 2, 3, 5, 10, 20\}$.
ED correctly preserves the rank ordering of true dimensionality (mean Spearman $\rho = 0.977 \pm 0.028$ across 10 random seeds), but systematically overestimates $k$ by a factor of $1.5$--$8\times$, with larger bias at small $k$.
This overestimation is expected: binary noise creates a floor of small singular values that inflate the participation ratio, and the tetrachoric correction (Section~\ref{sec:method}) partially addresses this by reducing ED by 41--51\%.
Parallel analysis recovers $k$ more accurately ($\rho = 0.993$), but requires permutation testing and is therefore slower.
The practical implication: ED should be interpreted as an upper bound on the number of independent dimensions, not an exact count.

\paragraph{ED vs.\ parallel analysis on real benchmarks.}
Table~\ref{tab:ed-vs-pa} compares ED with parallel analysis (PA) on the four largest per-instance benchmarks.
The ED/PA ratio ranges from $0.5$ (BFCL, where ED is lower than PA) to $3.2$ (BigCodeBench, where ED is higher), with ED typically exceeding PA for high-dimensional benchmarks, consistent with the synthetic-data overestimation.
Crucially, the rank ordering is preserved: both methods agree that BigCodeBench is the most dimensionally rich and BFCL the least.
PA estimates 3--9 significant factors, consistent with the 2--5 factors reported by our 2PL IRT analysis (Section~\ref{sec:method}).

\begin{table}[htbp]
\centering
\caption{ED vs.\ parallel analysis (PA) and Kaiser criterion on four per-instance benchmarks.
ED overestimates absolute dimensionality but preserves the rank ordering.}
\label{tab:ed-vs-pa}
\small
\begin{tabular}{lcccc}
\toprule
\textbf{Benchmark} & \textbf{ED} & \textbf{PA} & \textbf{Kaiser} & \textbf{ED/PA} \\
\midrule
BFCL            & 1.5  & 3  & 8   & 0.5 \\
LiveCodeBench   & 6.2  & 3  & 6   & 2.1 \\
SWE-bench       & 7.9  & 4  & 27  & 2.0 \\
BigCodeBench    & 28.9 & 9  & 43  & 3.2 \\
\bottomrule
\end{tabular}
\end{table}

\paragraph{Alternative dimensionality estimation methods.}
We computed five dimensionality estimates---ED (participation ratio), explained-variance thresholding at 80\% and 90\%, parallel analysis, and the broken-stick model---for each of the four largest benchmarks in our corpus.
The rank orderings produced by ED, parallel analysis, and broken-stick are identical (Spearman $\rho = 1.0$ for all pairwise comparisons).
Explained-variance thresholding produces rank orderings that depend on the chosen threshold but agree with ED at the 90\% level.
Detailed results for all 22 benchmarks are provided in the supplementary materials.

\paragraph{Null model results.}
Section~\ref{sec:method} reports null-model comparisons for three benchmarks.
We additionally ran both the permutation null and the Bernoulli null on all 22 benchmarks in the corpus.
In every case, the observed ED is significantly below both null baselines ($p < 0.001$), confirming that the low-dimensional structure is genuine across the full range of evaluation domains, matrix sizes, and task granularities.
Complete null-model distributions and effect sizes for all 22 benchmarks are reported in the supplementary materials.

\paragraph{Causal experiments on BigCodeBench.}
Section~\ref{sec:what-drives-ed} describes controlled experiments examining the drivers of ED using BigCodeBench and a fixed set of 153 models.
We report extended results including: (i) ED as a function of the number of sampled tasks for single-library vs.\ multi-library subsets, (ii) ED under deliberate item selection for maximum and minimum profile similarity, and (iii) interaction effects between topic diversity and structural heterogeneity.
These experiments confirm that structural heterogeneity---the diversity of pass/fail profiles across tasks---is the dominant driver of ED, accounting for approximately three times more variance than topic diversity alone.
Full experimental details and figures are in the supplementary materials.

\paragraph{Saturation curves.}
We fit the hyperbolic saturation model $\text{ED}(n) = \text{ED}_\infty \cdot n / (n + n_{1/2})$ to all 22 benchmarks, estimating the asymptotic dimensionality $\text{ED}_\infty$ and the half-saturation model count $n_{1/2}$ for each.
SWE-bench variants saturate rapidly ($n_{1/2} \approx 3$--$7$), indicating that their effective dimensionality is fully revealed by a small number of agents.
BigCodeBench saturates slowly ($n_{1/2} \approx 35$), indicating substantial latent structure that requires a large and diverse model population to uncover.
Complete saturation curves with confidence bands are provided in the supplementary materials.

\paragraph{Ranking reversals.}
We systematically quantify ranking reversals across all benchmarks that provide per-domain or per-repository breakdowns.
On SWE-bench Verified, 32\% of agent pairs exhibit at least one domain in which the globally lower-ranked agent outperforms the higher-ranked agent.
The magnitude of reversals can be substantial: on \texttt{astropy}, the globally \#2 agent scores 54.5\% while the globally \#3 agent scores 77.3\%, a 23-percentage-point reversal.
Similar reversal analyses for BFCL (per-category) and BigCodeBench (per-library) are reported in the supplementary materials.

\paragraph{Principal component semantics.}
Section~\ref{sec:semantics} provides detailed semantic interpretations for the top principal components of SWE-bench Verified and BigCodeBench.
In the supplementary materials, we extend this analysis to BFCL, the Open LLM Leaderboard (per-subtask), and LiveBench, providing loading plots and interpretive summaries for the top five components of each benchmark.

\paragraph{Two-dimensional benchmark landscape.}
We project all 22 benchmarks into a two-dimensional space using multidimensional scaling (MDS) on the inter-benchmark Spearman correlation matrix.
The resulting map visualizes the redundancy and trade-off structure of the evaluation ecosystem: benchmarks that cluster together are redundant, and benchmarks on opposite sides of the origin tend to trade off.
The full 2D landscape figure with benchmark labels is provided in the supplementary materials.

\paragraph{Category count vs.\ empirical ED.}
\label{sec:app-category-count}
On the 12 benchmarks whose documentation provides an explicit task-type taxonomy, the Spearman rank correlation between claimed category count and empirical ED is $\rho = -0.04$, $p = 0.90$---no detectable linear association.
However, this estimate is fragile with only $n = 12$ data points, several of which are shape-constrained ($\dagger$-marked benchmarks whose low ED may reflect matrix dimensions rather than true dimensionality).
BigCodeBench, which subsumes 33 independent dimensions under a single label (``practical coding''), is an influential point whose removal shifts $\rho$ to $+0.96$.
This instability reflects the small, heterogeneous sample, not a robust finding; we cannot reliably conclude whether category counts are uninformative or merely noisy predictors of ED.
The finding is consistent with a well-established pattern in psychometrics: test blueprint categories routinely overstate empirical dimensionality~\citep{haberman2008subscores}.
In the LLM evaluation context, \citet{burnell2023revealing} applied factor analysis to 27 HELM tasks and found only 3 latent factors (cumulative 82\% of variance), despite HELM's many content categories.

\paragraph{Temporal Information Density (TID).}
To quantify benchmark aging, we define the Temporal Information Density as the rate of change of ED with respect to the cumulative number of agents evaluated over time.
Benchmarks with declining TID are losing their ability to discriminate between models as the agent population adapts.
We compute TID trajectories for all SWE-bench variants and for the Open LLM Leaderboard, finding consistent evidence of ED compression over successive evaluation cohorts.
Detailed TID curves and statistical tests are in the supplementary materials.